\documentclass[10pt]{article} 
\usepackage[preprint]{tmlr}

\usepackage{hyperref}
\usepackage{url}

\usepackage{booktabs}       
\usepackage{amsfonts}       
\usepackage{nicefrac}       
\usepackage[dvipsnames]{xcolor}
\usepackage{colortbl}
\usepackage{natbib}
\usepackage{lipsum}
\usepackage{multirow}
\usepackage{rotating}

\usepackage{mathtools,xparse}
\usepackage{amsmath}
\usepackage{def}
\usepackage{amsthm}
\usepackage{thmtools}
\usepackage{thm-restate}
\usepackage{graphicx}
\usepackage{algorithm}
\usepackage{algorithmicx}
\usepackage{tikz,pgfplots,pgf}
\usetikzlibrary{matrix,shapes,arrows,positioning}
\usepackage{subcaption}
\usepackage{soul}
\usepackage{listings}
\usepackage[noend]{algpseudocode}
\pgfplotsset{compat=1.17} 
\usepackage{bbm}
\usepackage[capitalize,noabbrev]{cleveref}
\usepackage{graphicx}

\newtheorem{heuristic}{Heuristic}
\newtheorem*{theorem*}{Theorem}
\newtheorem{theorem}{Theorem}

\newtheorem{lemma}[theorem]{Lemma}

\theoremstyle{definition}
\newtheorem*{definition}{Definition}

\newenvironment{proofsketch}{%
  \proof}{\endproof}

\newenvironment{proofintuition}{%
  \proof}{\endproof}

\title{A Framework for PAC-Bayes Derandomization: \\Applications to Majority Votes}

\author{\name Benjamin Leblanc \email benjamin.leblanc.2@ulaval.ca \\
      \addr \textit{Département d’informatique et de génie logiciel}\\
      \textit{Université Laval}
      \AND
      \name Pascal Germain \email pascal.germain@ift.ulaval.ca \\
      \addr \textit{Département d’informatique et de génie logiciel}\\
      \textit{Université Laval}}

\def\month{MM}  
\def\year{YYYY} 
\def\openreview{\url{https://openreview.net/forum?id=XXXX}} 

\begin{document}

\maketitle

\begin{abstract}
PAC-Bayes is a popular and efficient framework for obtaining generalization guarantees in situations involving uncountable hypothesis spaces. Unfortunately, in its classical formulation, it only provides guarantees on the expected risk of a randomly sampled hypothesis. This requires stochastic predictions at test time, making PAC-Bayes unusable in many practical situations where a single deterministic hypothesis must be deployed. We propose a unified framework for extracting guarantees for a single hypothesis from PAC-Bayesian guarantees. We present a general derandomization framework, and derive from it a numerical bound for majority votes and for three different probability distributions over the posterior space. We study the empirical behaviour of our approach, showing that it consistently outperforms popular baselines in terms of generalization bounds for single classifiers (by up to a factor of 2), but sometimes requires a sacrifice in accuracy.
\end{abstract}

\section{Introduction}
\label{section:introduction}
The PAC-Bayes theory, initiated by \citet{mcallester1998some,DBLP:journals/ml/McAllester03} and enriched by many \citep{DBLP:journals/ftml/Alquier24}, has become a prominent framework for obtaining non-vacuous generalization guarantees on uncountable hypothesis spaces. Such hypothesis spaces include linear classifiers \citep{DBLP:conf/nips/LangfordS02, DBLP:conf/icml/GermainLLM09}, weighted majority vote \citep{DBLP:conf/icml/RoyLM11, DBLP:journals/ml/BelletHMS14, DBLP:conf/nips/ZantedeschiVMEH21} and neural networks \citep{DBLP:conf/uai/DziugaiteR17,DBLP:conf/nips/LetarteGGL19, DBLP:journals/corr/abs-2109-10304, DBLP:journals/jmlr/Perez-OrtizRSS21, DBLP:conf/icml/LeblancBDDG25}.

In contrast to classical PAC bounds \citep{DBLP:journals/cacm/Valiant84}, such as those based on the concept of VC-dimension \citep{DBLP:books/daglib/0026015} or Rademacher complexity \citep{DBLP:books/daglib/0034861}, PAC-Bayes bounds involve prior and posterior distributions over the hypothesis space, the latter usually being learned by minimizing a generalization guarantee as a training objective \citep{DBLP:conf/icml/GermainLLM09, DBLP:journals/jmlr/Parrado-HernandezASS12, DBLP:conf/nips/WuMLIS21}. Classical PAC-Bayesian guarantees do not hold for every single hypothesis simultaneously or for a single one, but bound the expected risk of a randomly drawn hypothesis with respect to the posterior distribution. This provides aggregate information about the hypothesis space weighted by the learned posterior distribution. In many applications, a single hypothesis must be used, for which classical PAC-Bayes cannot provide generalization guarantees. For instance, in breast cancer diagnosis \citep{NAJI2021487}, a patient must receive a single, consistent diagnosis; randomly varying predictions would erode trust and violate medical protocols. Similarly, in situations where interpretability is of importance, say for fairness concerns \citep{blog4, blog5, blog6}, stochasticity could obfuscate the prediction mechanism.

A notable derivation of the PAC-Bayesian framework is found in disintegrated PAC-Bayes bounds \citep{DBLP:conf/colt/BlanchardF07, catoni-ori, DBLP:conf/nips/RivasplataKSS20, viallard2024general}, which bound the deterministic risk of any classifier \textit{randomly} drawn according to the posterior distribution $Q$. Such bounds apply to the generalization risk of a single classifier, yet this classifier is randomly drawn from the posterior distribution rather than deterministically chosen. Another line of work is actually interested in bounding the risk of a single hypothesis using PAC-Bayes bounds \cite{DBLP:conf/nips/MasegosaLIS20, DBLP:journals/jmlr/Shawe-TaylorH09, DBLP:conf/pkdd/LacasseLMT10, DBLP:conf/nips/LacasseLMGU06}, but usually comes with limitations. For example: the distribution over the hypothesis space must be either Categorical or Gaussian, the hypothesis whose risk is bounded cannot be chosen, etc.


\paragraph{Contributions.} 
We propose a unified framework for obtaining generalization bounds on the risk of single hypotheses using PAC-Bayes bounds on a weighted hypothesis space as a building block. This conversion, which we call \textit{Derandomization}, leverages the tightness of PAC-Bayes bounds to address concerns about their practical use. This is achieved by deriving a conversion bound for majority-vote classifiers, using Categorical, Dirichlet, or Gaussian distributions. Finally, we evaluate (and compare to several baselines) our approach on many binary and multi-class classification tasks, providing empirical evidence of its tightness; albeit achieving a better bound sometimes requires sacrificing accuracy.

\section{Background and definitions}
\label{section:background}
\paragraph{Prediction problem.}
A dataset $S~=~\{(\xbf_j, \ybf_j)\}_{j=1}^{m}$ consists of $m$ examples, each being a feature-label pair $(\xbf, \ybf) \in \Xcal\times\Ycal$. We focus on (possibly multivariate) classification tasks, denote $k$ as the number of classes, and let $\Ycal$ be the set of the $k$ different one-hot vectors of length $k$. A predictor is a function $\hbf : \mathcal{X} \rightarrow \mathcal{Y}'$, from a predictor space $\Hcal$, with $\Ycal'\supseteq\Ycal$. We consider a binary loss function $\ell : \Hcal\times\mathcal{X} \times \mathcal{Y} \rightarrow \{0,1\}$. We denote~$\Dcal$ the data-generating distribution over $\Xcal \times \Ycal$ such that $S \sim \Dcal^m$.

Given a predictor $\hbf\in\Hcal$ and a loss function $\ell$, the empirical loss of the predictor over a set of $m$ \emph{i.i.d.} examples is 
$$
\risk_{S}(\hbf) = \frac{1}{m}\sum_{(\xbf,\ybf)\in S} \ell(\hbf,\mathbf{x}, \ybf),
$$ 
while the generalization loss (\textit{risk}) of a predictor~$\hbf$ is 
$$
\risk_{\mathcal{D}}(\hbf) = \underset{(\mathbf{x}, \ybf)\sim \Dcal}{\mathbb{E}}\left[\ell(\hbf,\mathbf{x}, \ybf)\right].
$$

\subsection{PAC-Bayesian learning framework}

A defining characteristic of PAC-Bayes bounds is that they rely on prior $P$ and posterior $Q$ distributions over the predictor space $\Hcal$. Hence, most PAC-Bayes results are expressed as upper bounds on the $Q$-expected loss of the predictor space ($\risk_\Dcal(Q)=\Ebb_{\hbf \sim Q} \risk_{\Dcal}(\hbf)$, referred to as the \textit{Gibbs risk}). A typical PAC-Bayes bound states that for any prior $P$, with probability at least $1-\delta$, for all posteriors $Q$ over $\Hcal$:
$$
\risk_{\Dcal}(Q) \leq f\Big(\risk_S(Q), \KL(Q||P), m, \delta\Big),
$$
for some function $f$, where $\risk_S(Q) = \Ebb_{\hbf \sim Q}\risk_{S}(\hbf)$.

\section{Derandomization bounds}
\label{section:DRND_bounds}
We formalize the capacity to bound the risk of a single classifier with the help of a bound on the Gibbs risk by introducing \textit{Derandomization} bounds.
\begin{definition}[Derandomization bound]
    For some hypothesis space $\Hcal$, distribution $Q$ over $\Hcal$, and hypothesis $\hbf\in\Hcal$, a Derandomization bound is an upper bound on $\risk_{\mathcal{D}}(\hbf)$ which depends on $\risk_\Dcal(Q)$ and takes the following general form, for some function $f$:
    $$
    \risk_\Dcal(\hbf) \leq f\Big(\risk_\Dcal(Q)\Big).
    $$
\end{definition}

\subsection{A general bound}\label{subsec:general}

We first provide a general relationship between the risk of a single classifier and the Gibbs risk. To do so, we first decompose the Gibbs risk into two terms by conditioning on whether a given hypothesis $\hbf$ makes an error:
\begin{align*}
\CondZero_{\Dcal}^Q(\hbf) &= \underset{(\mathbf{x}, \ybf)\sim \Dcal}{\mathbb{E}} \left[\underset{\hbf'\sim Q}{\mathbb{E}} \ell(\hbf',\mathbf{x}, \ybf)~\bigg|~\ell(\hbf,\mathbf{x}, \ybf) = 0\right],\\
\CondOne_{\Dcal}^Q(\hbf) &= \underset{(\mathbf{x}, \ybf)\sim \Dcal}{\mathbb{E}} \left[\underset{\hbf'\sim Q}{\mathbb{E}} \ell(\hbf',\mathbf{x}, \ybf)~\bigg|~\ell(\hbf,\mathbf{x}, \ybf) = 1\right].
\end{align*}
Here, $\CondZero_{\Dcal}^Q(\hbf)$ can be seen as the Gibbs risk on $\hbf$'s correct predictions, and $\CondOne_{\Dcal}^Q(\hbf)$ the Gibbs risk on $\hbf$'s incorrect predictions.\footnote{We discuss how to generalize these quantities to any (possibly unbounded) loss in \cref{section:conclusion}.}
Intuitively, if $\mathbb{E}_{\hbf'\sim Q} \ell(\hbf',\mathbf{x},\ybf)$ is positively correlated with $\ell(\hbf,\mathbf{x}, \ybf)$, then we should have that $\CondZero_{\Dcal}^Q(\hbf) \leq \CondOne_{\Dcal}^Q(\hbf)$.

\begin{restatable}[Stochastic-deterministic relation]{proposition}{general}\label{prop:general}
    For any data distribution $\Dcal$, distribution $Q$ over~$\mathcal{H}$ and classifier $\hbf\in\mathcal{H}$ such that $\CondOne_{\Dcal}^Q(\hbf) \neq \CondZero_{\Dcal}^Q(\hbf)$,
we have
$$
\risk_{\Dcal}(\hbf) = \frac{\risk_\Dcal(Q) - \CondZero_{\Dcal}^Q(\hbf)}{\CondOne_{\Dcal}^Q(\hbf) - \CondZero_{\Dcal}^Q(\hbf)}.
$$
\end{restatable}
See \cref{appendix:proofs} for complete demonstrations of all mathematical results presented in the article. 
\begin{proofsketch}
We have
$
\risk_\Dcal(Q) = \CondOne_{\Dcal}^Q(\hbf) \cdot \risk_{\Dcal}(\hbf)+ \CondZero_{\Dcal}^Q(\hbf) \cdot \left(1 - \risk_{\Dcal}(\hbf)\right).
$
Solving for $\risk_{\Dcal}(\hbf)$ leads to the main result.
\end{proofsketch}

We emphasize that \cref{prop:general} holds for any $\hbf\in\Hcal$, while in the literature, only one specific $\hbf$ leads to a closed-form relation between $\risk_{\Dcal}(\hbf)$ and $\risk_\Dcal(Q)$ (see \cref{section:related}).

Classical PAC-Bayesian tools bound $\risk_{\Dcal}(Q)$ with respect to the distribution~$\Dcal$. However, $\CondZero_{\Dcal}^Q(\hbf)$ and $\CondOne_{\Dcal}^Q(\hbf)$ are related to conditional distributions $\Dcal|\{\ell(\hbf,\mathbf{x}, \ybf) = 0\}$ and $\Dcal|\{\ell(\hbf,\mathbf{x}, \ybf) = 1\}$, which are dependent on $\hbf\in\Hcal$. Therefore, existing PAC-Bayesian bounds cannot be used directly for bounding those quantities. \cref{prop:conditional_PB} introduces a modification of the well-known Seeger's bound \citep{DBLP:journals/jmlr/Seeger02} (see \cref{theo:classical_kl} in appendix) that bounds both $\CondZero_{\Dcal}^Q(\hbf)$ and $\CondOne_{\Dcal}^Q(\hbf)$.

\begin{restatable}[Conditional PAC-Bayes]{theorem}{conditional}\label{prop:conditional_PB}
    Let \mbox{$\kl(q, p) \,{=}\, q \ln\!\left(\frac{q}{p}\right) + (1 {-} q) \ln\!\left(\frac{1-q}{1-p}\right)$}. For any data distribution~$\Dcal$, hypothesis $\hbf\in\Hcal$, prior distribution $\Pcal$ over $\Hcal$, and $\delta\in (0,1]$, both the following statements hold with probability at least $1-\delta$ over the draw $S\sim\Dcal^m$, for all $Q$ over $\Hcal$:
\begin{align*}
&\kl\left(\CondZero^Q_{S}(\hbf)\Big|\Big|\CondZero^Q_{\Dcal}(\hbf)\hspace{-0.5mm}\right)\hspace{-0.5mm}\leq \hspace{-0.5mm}\frac{1}{m_{(\hbf,0)}}\hspace{-0.5mm}\left[\KL(Q||P)\hspace{-0.5mm}+\hspace{-0.5mm}\ln\hspace{-0.5mm}\left(\frac{2\sqrt{m}}{\delta}\right)\right],\\
&\kl\left(\CondOne^Q_{S}(\hbf)\Big|\Big|\CondOne^Q_{\Dcal}(\hbf)\hspace{-0.5mm}\right)\hspace{-0.5mm}\leq \hspace{-0.5mm}\frac{1}{m_{(\hbf,1)}}\hspace{-0.5mm}\left[\KL(Q||P)\hspace{-0.5mm}+\hspace{-0.5mm}\ln\hspace{-0.5mm}\left(\frac{2\sqrt{m}}{\delta}\right)\right],
\end{align*}
where $\CondZero^Q_{S}(\hbf) {=} \risk_{S_{(\hbf,0)}}(Q)$, $\CondOne^Q_{S}(\hbf) {=} \risk_{S_{(\hbf,1)}}(Q)$, 
$S_{(\hbf,\cdot)} {=} \{(\xbf,\ybf){\in} S:\ell(\hbf,\xbf,\ybf){=}\cdot\}$, and 
$|S_{(\hbf,\cdot)}|{=}m_{(\hbf,\cdot)}$.
\end{restatable}

Note that if either $m_{(\hbf,0)}$ or $m_{(\hbf,1)}$ is zero, the bound value on $\CondZero^Q_{\Dcal}(\hbf)$ or $\CondOne^Q_{\Dcal}(\hbf)$ is not computable, and one can consider a trivial bound value of $1$. 

The following result uses a result such as \cref{prop:conditional_PB}, and any more classical PAC-Bayes bound such as \cref{theo:classical_kl}, to bound the quantities involved in the right-hand side of the equality in \cref{prop:general}. In other words, it bounds the risk of any hypothesis $\hbf$ using three PAC-Bayesian bounds simultaneously.

\begin{restatable}[Triple bound -- Single hypothesis]{corollary}{triplesingle}\label{cor:triple-single}
    Let $\Dcal$ be a data distribution, $\hbf\in\Hcal$, and \mbox{$\delta_1, \delta_2, \delta_3\in[0,1]$} such that $\delta_1+ \delta_2+ \delta_3 < 1$. Let
\allowdisplaybreaks
\begin{align*}
&(1)\quad\underset{S\sim \Dcal^m}{\Pbb}\left(\forall Q\text{ over }\Hcal~:~\risk_\Dcal(Q)\leq \tildeR_S(Q)\right)\geq 1-\delta_1\,,\\
&(2)\quad\underset{S\sim \Dcal^m}{\Pbb}\left(\forall Q\text{ over }\Hcal~:~\CondZero_{\Dcal}^Q(\hbf)\geq \tildezero_S^Q(\hbf)\right)\geq 1-\delta_2\,,\\
&(3)\quad\underset{S\sim \Dcal^m}{\Pbb}\left(\forall Q\text{ over }\Hcal~:~\CondOne_{\Dcal}^Q(\hbf)\geq \tildeone_S^Q(\hbf)\right)\geq 1-\delta_3\,,
\end{align*}
hold for some quantities $\tildeR_S(Q)$, $\tildezero_S^Q(\hbf)$, and $\tildeone_S^Q(\hbf)$. Then:
$$
\underset{S\sim \Dcal^m}{\Pbb}\left(\forall Q\text{ over }\Hcal~:~\tildeone_S^Q(\hbf) > \tildezero_S^Q(\hbf)\Rightarrow\risk_{\Dcal}(\hbf) \leq \frac{\tildeR_S(Q) - \tildezero_S^Q(\hbf)}{\tildeone_S^Q(\hbf) - \tildezero_S^Q(\hbf)}\right)\geq 1-\delta\,,
$$
where $\delta = \delta_1+\delta_2+\delta_3$.
\end{restatable}
Notice that the condition $\CondOne_{\Dcal}^Q(\hbf) \neq \CondZero_{\Dcal}^Q(\hbf)$ from \cref{prop:general} does not appear here; it is replaced with the condition $\tildeone_S^Q(\hbf) > \tildezero_S^Q(\hbf)$. When using \cref{prop:conditional_PB} to ensure that (2) and (3) are respected in \cref{cor:triple-single}, one obtains the following form for $\tildezero_S^Q(\hbf)$ and $\tildeone_S^Q(\hbf)$:
\vspace{-2mm}
\begin{align*}
    \tildezero_S^Q(\hbf) &= \kl^{-1}\left(\CondZero^Q_{S}(\hbf)\bigg|\bigg|\frac{1}{m_{(\hbf,0)}}\hspace{-0.5mm}\left[\KL(Q||P)\hspace{-0.5mm}+\hspace{-0.5mm}\ln\hspace{-0.5mm}\left(\frac{2\sqrt{m}}{\delta}\right)\right]\right),\\
    \tildeone_S^Q(\hbf) &= \kl^{-1}\left(\CondOne^Q_{S}(\hbf)\bigg|\bigg|\frac{1}{m_{(\hbf,1)}}\hspace{-0.5mm}\left[\KL(Q||P)\hspace{-0.5mm}+\hspace{-0.5mm}\ln\hspace{-0.5mm}\left(\frac{2\sqrt{m}}{\delta}\right)\right]\right),
\end{align*}
where $\kl^{-1}(p||b) = \sup\{p\in[0,1]:\kl(q||p)\leq b\}$. Note that $m_{(\hbf,1)} = m\cdot R_S(\hbf)$ and $m_{(\hbf,0)} = m(1-R_S(\hbf))$; in order for each concentration inequality in \cref{cor:triple-single} to be tight, the training risk $R_S(\hbf)$ must neither be too small or too big.

\cref{cor:triple-single} only holds for a hypothesis $\hbf$ chosen before seeing $S$. Thus, it cannot be used to select~$\hbf$ based on $S$, as the probabilistic bounds on $\CondZero^Q_\Dcal(\hbf)$ and $\CondOne^Q_\Dcal(\hbf)$ would not be valid. \cref{cor:triple} addresses this by replacing probabilistic  lower bounds on $\CondZero^Q_\Dcal(\hbf)$ and $\CondOne^Q_\Dcal(\hbf)$ with deterministic lower bounds that hold for all $\hbf$ simultaneously. This enables data-dependent hypothesis selection (\cref{section:DRND_mv}).

\begin{restatable}[Triple bound]{corollary}{triple}\label{cor:triple}
    Let $\Dcal$ be a data distribution and \mbox{$\delta\in[0,1]$}. Let
$$
\underset{S\sim \Dcal^m}{\Pbb}\Big(\forall Q\text{ over }\Hcal~:~\risk_{\Dcal}(Q)\leq \tildeR_S(Q)\Big)\geq 1-\delta.
$$
hold for some quantity $\tildeR_S(Q)$. Then,\\

\centerline{\resizebox{\linewidth}{!}{$\displaystyle\underset{S\sim \Dcal^m}{\Pbb}\left(\forall Q\text{ over }\Hcal,
\hbf\in\Hcal~:~\CondZero_{\Dcal}^Q(\hbf)\geq \tildezero_S^Q(\hbf) \land \CondOne_{\Dcal}^Q(\hbf) \geq \tildeone_S^Q(\hbf) \land \tildeone_S^Q(\hbf) > \tildezero_S^Q(\hbf)\Rightarrow \risk_{\Dcal}(\hbf) \leq \frac{\tildeR_S(Q) - \tildezero_S^Q(\hbf)}{\tildeone_S^Q(\hbf) - \tildezero_S^Q(\hbf)}\right)\geq1-\delta\,.$}}
\end{restatable}

In the following section, we restrict ourselves to given hypothesis sets $\Hcal$ that enable direct computation of $ \tildeR_S(Q),~\tildezero_S^Q(\hbf)$ and $ \tildeone_S^Q(\hbf)$, so that \cref{cor:triple} can be used as a training objective. In short, we seek tight lower bounds $\tildezero_S^Q(\hbf)$ and $\tildeone_S^Q(\hbf)$ on $\CondZero_\Dcal^Q(\hbf)$ and $\CondOne_\Dcal^Q(\hbf)$ that ensure $\tildeone_S^Q(\hbf) > \tildezero_S^Q(\hbf)$.

\section{Derandomization bounds for stochastic majority vote}
\label{section:DRND_mv}
Weighted majority vote is a central technique for combining predictions of multiple classifiers, for example in random forests \citep{DBLP:journals/ml/Breiman96b, DBLP:journals/ml/Breiman01}, boosting \citep{DBLP:conf/icml/FreundS96}, gradient boosting \citep{DBLP:conf/nips/MasonBBF99, sto-grad-boost}, and when combining predictions of heterogeneous classifiers. It is part of the winning strategies in many machine learning competitions. The power of the majority vote is in the cancellation of errors effect: when the errors of individual classifiers are independent or anticorrelated, and the error probability of individual classifiers is smaller than 0.5, then the errors average out, and the majority vote tends to outperform the individual classifiers \citep{DBLP:conf/nips/MasegosaLIS20}. Furthermore, weighted majority vote remains an active topic in PAC-Bayesian research \citep{DBLP:journals/ml/LorenzenIS19, DBLP:conf/nips/WuMLIS21,DBLP:conf/nips/MasegosaLIS20, DBLP:conf/pkdd/ViallardGHM21, DBLP:conf/icml/AbbasA22, DBLP:conf/nips/0003S22, viallard2024general, DBLP:journals/corr/abs-2411-06276-hennequin}.

Given a finite set of $n$ base classifiers \mbox{$\mathcal{F} = \{\fbf_1, \dots, \fbf_n\}$}, where $\fbf_i:\mathcal{X}\rightarrow\mathcal{Y}$, let $\Hcal$ be the space of possible majority vote classifiers: \mbox{$\Hcal=\left\{\sum_{i=1}^n p_i\fbf_i~|~\mathbf{p}\in\mathcal{W}\right\}$}, where $\mathcal{W}\subseteq \mathbb{R}^n$ is the majority vote weight space. We denote \mbox{$\hbf_{\mathbf{w}}(\mathbf{x}) = \sum_{i=1}^n w_i\fbf_i(\mathbf{x})=\mathbf{w}\cdot\mathbf{f}(\mathbf{x})\in\Ycal'$} the deterministic majority vote with weights $\wbf$. Since each majority vote is uniquely determined by its weight vector, we define $Q$ over the weight space $\Wcal$, writing $\wbf\sim Q$ instead of $h_\wbf\sim Q$.

We now consider three families of posterior distributions, each parameterized by $\pbf\in\Rbb^n$, which admit different sets of realizations $\Wcal$: Categorical $\Ccal(\pbf)$, where $\Wcal=\big\{\mathbf{w}\in\{0,1\}^n\,\big|\,\sum_{i=1}^n w_i = 1\big\}$; Dirichlet $D(\pbf)$, where $\Wcal=\big\{\mathbf{w}\in[0,1]^n \,\big|\,\sum_{i=1}^n w_i = 1\big\}$; and Unit-variance Gaussian $\Ncal(\pbf, \Ibf_{n\times n})$, where $\Wcal=\mathbb{R}^n$. For each distribution, we bound the risk of the majority vote weighted by $\pbf$. For example, for the Categorical distribution with parameters $\mathbf{p}$, that is, $\Ccal(\pbf)$, we end up bounding the risk of $\hbf_{\mathbf{p}}$. We will see that this particular choice of majority vote leads to interesting forms for $\CondZero_{\Dcal}^Q(\hbf_\pbf)$ and $\CondOne_{\Dcal}^Q(\hbf_\pbf)$ (thus enabling their lower bounding with probability $1$). This choice is also motivated by the fact that $\risk_\Dcal(\hbf_\pbf)$ is most likely positively correlated with $\risk_\Dcal(Q)$; if $\risk_\Dcal(Q)$ is small, then $\risk_\Dcal(\hbf_\pbf)$ will most likely be small as well.

\subsection{Categorical assumption}\label{subsec:categorical}

In this section, we will only meet positive weights $\wbf$ that sum to one. Therefore, to simplify calculations, we turn to the following loss function:
$$
\forall\mathbf{w}\in\left\{\wbf'\in[0,1]^n~\bigg|~\sum_{i=1}^n w_i' = 1\right\},(\xbf,\ybf)\in\Xcal\times\Ycal~:~\ell(\hbf_\wbf,\mathbf{x}, \ybf) = \mathbbm{1}\left\{\sum_{i=1}^n w_i\mathbbm{1}\left\{f_i(\mathbf{x}) \neq \ybf\right\} \geq 0.5\right\}.
$$
If at least half of the total weight (that is, $0.5$) is assigned to wrong predictions, then an error occurs for the majority vote as a whole.

\begin{restatable}[Majority vote--Categorical derandomization]{proposition}{categorical}\label{prop:categorical}
In the context of \cref{prop:general}: let $Q = \mathcal{C}(\mathbf{p})$ be a Categorical distribution with parameters $\mathbf{p}\in\{\pbf'\in[0,1]^n~|~\sum_{i=1}^n p_i' = 1\}$. For any data distribution $\Dcal$:
\begin{align*}
&\forall(\xbf,\ybf)\in\Xcal\times\Ycal~:~\underset{\wbf\sim \mathcal{C}(\mathbf{p})}{\mathbb{E}} \ell(\hbf_{\wbf},\mathbf{x}, \ybf)=e(h_\pbf, \xbf,\ybf),\\
&\CondZero_{\Dcal}^{\mathcal{C}(\mathbf{p})}(\hbf_{\mathbf{p}}) = \underset{(\mathbf{x}, \ybf)\sim \Dcal}{\mathbb{E}}\left[e(h_\pbf, \xbf,\ybf)~|~e(h_\pbf, \xbf,\ybf) < 0.5\right],
 \\
&\CondOne_{\Dcal}^{\mathcal{C}(\mathbf{p})}(\hbf_{\mathbf{p}}) = \underset{(\mathbf{x}, \ybf)\sim \Dcal}{\mathbb{E}}\left[e(h_\pbf, \xbf,\ybf)~|~e(h_\pbf, \xbf,\ybf)\geq 0.5\right],
\end{align*}
where $e(h_\pbf, \xbf,\ybf)=\sum_{i=1}^n p_i \mathbbm{1}\{\fbf_i(\mathbf{x}) \neq \ybf\}$.
\end{restatable}

Note that $\CondOne_{\Dcal}^{\mathcal{C}(\mathbf{p})}(\hbf_{\mathbf{p}}) > \CondZero_{\Dcal}^{\mathcal{C}(\mathbf{p})}(\hbf_{\mathbf{p}})$ is ensured by their respective forms. \cref{prop:categorical} makes it explicit that $\CondOne_{\Dcal}^{\mathcal{C}(\mathbf{p})}(\hbf_{\mathbf{p}}) \geq 0.5$ and \mbox{$\CondZero_{\Dcal}^{\mathcal{C}(\mathbf{p})}(\hbf_{\mathbf{p}}) \geq 0$}, leading to the following worst-case scenario, when substituted in \cref{prop:general}:
$$
\forall\hbf\in\Hcal~:~\risk_{\Dcal}(\hbf) \leq \frac{\risk_{\Dcal}(Q) - 0}{0.5 - 0} = 2\risk_{\Dcal}(Q).
$$
This recovers the classical "factor-2" bound from \cite{DBLP:conf/nips/LangfordS02} as a worst case. 

To obtain the best upper bound on $\risk_{\Dcal}(\hbf)$, we need the best deterministic lower bound on both $\CondZero_{\Dcal}^{\mathcal{C}(\mathbf{p})}(\hbf_{\mathbf{p}})$ and $\CondOne_{\Dcal}^{\mathcal{C}(\mathbf{p})}(\hbf_{\mathbf{p}})$. This in turn corresponds to finding the smallest possible value for $e(h_\pbf, \xbf,\ybf)$, knowing $e(h_\pbf, \xbf,\ybf)$ is respectively at least $0$ and $0.5$ (see \cref{prop:categorical}). While the former (smallest amount of weight assigned to base classifiers making correct predictions) can usually be found in a reasonable computing time for majority vote classifiers, the latter (smallest amount of weight assigned to base classifiers making correct predictions, given that it is at least $0.5$) is trickier to identify while remaining computationally tractable. Indeed, given our definition $e(h_\pbf, \xbf,\ybf)=\sum_{i=1}^n p_i \mathbbm{1}\{f_i(\mathbf{x}) \neq \ybf\}$, each base classifier can either make an error or not, affecting the resulting value for $e(h_\pbf, \xbf,\ybf)$; and since there are $n$ base classifiers, it means there are $O(2^n)$ possible values for $e(h_\pbf, \xbf,\ybf)$. 

To efficiently find the smallest value $e(h_\pbf, \xbf,\ybf)$ can take, given that it is bigger than $0.5$, we adopt the \textit{partition problem} perspective of our problem.

\begin{definition}[The Partition Problem]
    Given a set of numbers $\mathbf{a}$, the partition problem consists of finding 
    \small
    $$
    \underset{\mathbf{a}_1,\mathbf{a}_2}{\argmin}\left\{\Bigg|\sum_{a\in\mathbf{a}_1}a - \sum_{a'\in\mathbf{a}_2}a'\Bigg|:\{\abf_1,\abf_2\}\textup{ is a partition of }\abf\right\}.
    $$
\end{definition}

\begin{restatable}[Weights partitioning lower bound--Categorical]{proposition}{partitioncategorical}\label{prop:partition_categorical}
In the context of \cref{prop:categorical}: let $\mathbf{p}_1$ and~$\mathbf{p}_2$ be the result of the partition problem applied to~$\mathbf{p}$. Then,
$$
\CondOne_{\Dcal}^{\mathcal{C}(\mathbf{p})}(\hbf_{\mathbf{p}}) \geq \max\left(\sum_{p\in\mathbf{p}_1}p, \sum_{p'\in\mathbf{p}_2}p'\right).
$$
\end{restatable}

\begin{proofintuition}
The partition $\{\pbf_1, \pbf_2\}$ represents the 
worst-case scenario: voters' weights are split as evenly as possible between correct and incorrect predictions. Even in this worst case, when 
$\hbf_\pbf$ makes an error (that is, when $e(h_\pbf, \xbf,\ybf) \geq 0.5$), the weighted error fraction is at least $\max\big(\sum_{p\in\mathbf{p}_1}p, \sum_{p'\in\mathbf{p}_2}p'\big)$.
\end{proofintuition}

The partition problem, while NP-complete, has been extensively studied: there exists a pseudo-polynomial time dynamic programming algorithm \citep{10.1093/oso/9780195177374.003.0012}, making the computation of our lower bound tractable. We discuss more thoroughly in \cref{section:optimization} how the partition problem is solved efficiently using algorithms from the literature.

\subsection{Dirichlet assumption}\label{subsec:dirichlet}

In this section, we will only consider weights $\wbf\in\Rbb_{>0}^n$. Therefore, to simplify calculations, we turn to the following loss function:
$$
\forall\wbf\in\Rbb^n_{\geq0},(\xbf,\ybf)\in\Xcal\times\Ycal~:~\ell(\hbf_\mathbf{w},\mathbf{x}, \ybf) = \mathbbm{1}\left\{\sum_{i=1}^n w_i\mathbbm{1}\left\{\fbf_i(\mathbf{x}) \neq \ybf\right\} \geq \frac{||\mathbf{p}||_1}{2}\right\}.
$$
Once again, if at least half of the total weight (that is, $||\wbf||_1/2$) is assigned to wrong predictions, then an error occurs for the majority vote as a whole.

\begin{restatable}[Majority vote--Dirichlet derandomization]{proposition}{dirichlet}\label{prop:dirichlet}
In the context of \cref{prop:general}: let $Q = D(\pbf)$ be a Dirichlet distribution with parameters $\pbf\in\mathbb{R}_{>0}^n$. For any data distribution $\Dcal$:
\begin{align*}
&\forall(\xbf,\ybf)\in\Xcal\times\Ycal~:~\underset{\wbf\sim D(\mathbf{p})}{\mathbb{E}} \ell(\hbf_{\wbf},\mathbf{x}, \ybf)=I_{0.5}\big(||\pbf||_1-e(h_\pbf, \xbf,\ybf), \,e(h_\pbf, \xbf,\ybf)\big),\\
&\CondZero_{\Dcal}^{D(\mathbf{p})}(\hbf_{\mathbf{p}}) =\hspace{-2mm} \underset{(\mathbf{x}, \ybf)\sim \Dcal}{\mathbb{E}}\left[I_{0.5}\big(||\pbf||_1-e(h_\pbf, \xbf,\ybf), \,e(h_\pbf, \xbf,\ybf)\big)\Big|\,e(h_\pbf, \xbf,\ybf) < \frac{||\pbf||_1}{2}\right],\\
&\CondOne_{\Dcal}^{D(\mathbf{p})}(\hbf_{\mathbf{p}}) =\hspace{-2mm} \underset{(\mathbf{x}, \ybf)\sim \Dcal}{\mathbb{E}}\left[I_{0.5}\big(||\pbf||_1-e(h_\pbf, \xbf,\ybf), \,e(h_\pbf, \xbf,\ybf)\big)\Big|\,e(h_\pbf, \xbf,\ybf)\geq \frac{||\pbf||_1}{2}\right],
\end{align*}
where $e(h_\pbf, \xbf,\ybf)=\sum_{i=1}^n p_i \mathbbm{1}\{\fbf_i(\mathbf{x}) \neq \ybf\}$ and $I_x(\cdot,\cdot)$ is the regularized incomplete beta func. evaluated at $x$.
\end{restatable}

Notice, once again, that \mbox{$\CondOne_{\Dcal}^{D(\mathbf{p})}(\hbf_{\mathbf{p}}) > \CondZero_{\Dcal}^{D(\mathbf{p})}(\hbf_{\mathbf{p}})$} is ensured since $I_{0.5}(\cdot,\cdot)$ is decreasing in its first argument, and increasing in its second one. We also have that \mbox{$\CondOne_{\Dcal}^{D(\mathbf{p})}(\hbf_{\mathbf{p}}) \geq 0.5$} and $\CondZero_{\Dcal}^{D(\mathbf{p})}(\hbf_{\mathbf{p}}) \geq 0$, so that \mbox{$\risk_{\Dcal}(\hbf_{\mathbf{p}}) \leq 2\risk_\Dcal(Q)$} is the worst-case scenario. Finally, $\CondOne_{\Dcal}^{D(\mathbf{p})}(\hbf_{\mathbf{p}})$ can once again be lower-bounded in a way involving the partitioning problem, as per \cref{prop:partition_dirichlet}.

\begin{restatable}[Weights partitioning lower bound--Dirichlet]{proposition}{partitiondirichlet}\label{prop:partition_dirichlet}
    In the context of \cref{prop:dirichlet}: let $\mathbf{p}_1$ and~$\mathbf{p}_2$ be the result of the partition problem applied to $\mathbf{p}$. Let $\overset{\sim}{\pbf} = \max\left(\sum_{p\in\mathbf{p}_1}p, \sum_{p'\in\mathbf{p}_2}p'\right)$. Then,
$$
\CondOne_{\Dcal}^{D(\mathbf{p})}(\hbf_{\mathbf{p}}) \geq I_{0.5}\left(||\pbf||_1-\overset{\sim}{\pbf}, ~\overset{\sim}{\pbf}\right).
$$
\end{restatable}
As in the Categorical case, we lower-bound $\CondOne_{\Dcal}^{D(\mathbf{p})}(\hbf_{\mathbf{p}})$ by considering the most balanced partition of $\pbf$.

\subsection{Gaussian assumption}\label{subsec:gaussian}

In this section, we will meet weights $\wbf\in\Rbb^n$; we turn ourselves to the following loss function:
$$
\forall\wbf\in\Rbb^n,(\xbf,\ybf)\in\Xcal\times\Ycal~:~\ell(\hbf_\mathbf{w},\mathbf{x}, \ybf) = \mathbbm{1}\left\{\ybf \neq \argmax_{\hat{\ybf}\in\mathcal{Y}} \sum_{j=1}^n w_j \mathbbm{1}\{\fbf_j(\mathbf{x}) = \hat{\ybf}\}\right\}.
$$
We examine the binary classification and the multiclass classification setups independently, for they lead to different results.

\paragraph{Binary classification.}

Without loss of generality, let $\Ycal = \{-1, 1\}$.

\begin{restatable}[Majority vote--Binary\,Gaussian\,derandomization]{proposition}{binarygaussian}\label{prop:gaussian_binary}
In the context of \cref{prop:general}: let \mbox{$Q = \Ncal(\pbf, \I)$} be a Gaussian distribution with mean $\pbf\in\mathbb{R}^n$ and identity covariance matrix. For any $\Dcal$:
\begin{align*}
&\forall(\xbf,\ybf)\in\Xcal\times\Ycal~:~\underset{{\mathbf{w}\sim \Ncal\left(\mathbf{p}, \I\right)}}{\mathbb{E}} \ell(\hbf_{\mathbf{w}},\mathbf{x}, y)=\Phi\left(y\frac{\pbf\cdot\fbf(\xbf)}{||\fbf(\xbf)||}\right),\\
&\CondOne_{\Dcal}^{\Ncal(\mathbf{p}, \I)}(\hbf_{\mathbf{p}})=1-\underset{(\mathbf{x}, y)\sim \Dcal}{\mathbb{E}}\left[\Phi\left(\frac{|\pbf\cdot\fbf(\xbf)|}{||\fbf(\xbf)||}\right)\Bigg|\,y(\mathbf{p}\cdot\fbf(\xbf))\leq0\right],\\
&\CondZero_{\Dcal}^{\Ncal(\mathbf{p}, \I)}(\hbf_{\mathbf{p}}) =\underset{(\mathbf{x}, y)\sim \Dcal}{\mathbb{E}} \left[\Phi\left(\frac{|\pbf\cdot\fbf(\xbf)|}{||\fbf(\xbf)||}\right)\Bigg|\,y(\mathbf{p}\cdot\fbf(\xbf)) > 0\right],
\end{align*}
with \mbox{$\Phi(k) = \frac{1}{2}\left(1-\textup{erf}\big(\tfrac{k}{\sqrt{2}}\big)\right)$}, \mbox{$\erf(k) = \frac{2}{\sqrt{\pi}}\int_{0}^{k}e^{-t^2}dt$}.
\end{restatable}

Once again, both \mbox{$\risk_{\Dcal}(\hbf_{\mathbf{p}}) \leq 2\risk_\Dcal(Q)$} and \mbox{$\CondOne_{\Dcal}^Q(\hbf_{\mathbf{p}}) > \CondZero_{\Dcal}^Q(\hbf_{\mathbf{p}})$} 
hold. Since every base classifier has a prediction in $\{-1,+1\}$, for every $\xbf$, we have $||\fbf(\xbf)|| = \sqrt{n}$. We can lower-bound $\CondOne_{\Dcal}^Q(\hbf_{\mathbf{p}})$ as in \cref{subsec:categorical,subsec:dirichlet}.

\begin{restatable}[Weights partitioning lower bound--Binary Gaussian]{proposition}{partitionbinarygaussian}\label{prop:partition_binary_gaussian}
    In the context of \cref{prop:gaussian_binary}: let $\mathbf{p}_1$ and $\mathbf{p}_2$ be the result of the partition problem applied to~$\mathbf{p}$. Let $\overline{p} = \left|\sum_{p\in\mathbf{p}_1}p - \sum_{p'\in\mathbf{p}_2}p'\right|$. Then,
$$
\CondOne_{\Dcal}^{\Ncal(\mathbf{p}, \I)}(\hbf_{\mathbf{p}}) \geq 1- \Phi\left(\frac{\overline{p}}{\sqrt{n}}\right).
$$
\end{restatable}
Unlike the Categorical and the Dirichlet cases, where the lower bound $\CondZero_{\Dcal}^Q(\hbf_{\mathbf{p}}) \geq 0$ did not have a mathematical expression, the Gaussian structure allows a lower bound based on the partition problem.

\begin{restatable}[Weights maximizing lower bound--Binary Gaussian]{proposition}{maxbinarygaussian}\label{prop:max_binary_gaussian}
    In the context of \cref{prop:gaussian_binary}:
$$
\CondZero_{\Dcal}^{\Ncal(\mathbf{p}, \I)}(\hbf_{\mathbf{p}}) \geq \Phi\left(\frac{||\pbf||_1}{\sqrt{n}}\right).
$$
\end{restatable}

\paragraph{Multi-class classification.}

In practice, \cref{prop:categorical,prop:dirichlet,prop:gaussian_binary} are convenient provided that  $\mathbb{E}_{\hbf'\sim Q} \ell(\hbf',\mathbf{x}, \ybf)$ can be easily computed for any $(\xbf,\ybf)\in\Xcal\times\Ycal$. To our knowledge, for Gaussian distributions, an analytic expression exists in the PAC-Bayes literature for the case where $|\Ycal|=2$ only \citep{DBLP:conf/nips/LangfordS02}, which we used in the derivation of \cref{prop:gaussian_binary}. We generalize this result to the multi-class case by the following proposition. 

Without loss of generality, with $k$ classes, let $\Ycal$ be the set of $k$ one-hot vectors of length $k$. Note that $\fbf(\xbf)$ is an $n\times k$ matrix (we recall that $k$ is the number of classes). Here, $\fbf_i(\xbf)$ is the prediction of the $i^{\textup{th}}$ base classifier, whereas $\fbf_{:,i}(\xbf)$ is the $i^{\textup{th}}$ column of the matrix.

\begin{restatable}[Majority vote--Multivariate Gaussian stochastic risk]{proposition}{gaussian}\label{prop:gaussian_multi}
In the context of \cref{prop:general}: let $Q = \Ncal(\pbf, \I)$ be a Gaussian distribution with mean $\pbf\in\mathbb{R}^n$ and identity covariance matrix. Let $|\Ycal|=k$. For any data distribution $\Dcal$:
$$
\underset{{\mathbf{w}\sim \Ncal\left(\mathbf{p}, \I\right)}}{\mathbb{E}} \ell(\hbf_{\mathbf{w}},\mathbf{x}, \ybf) = 1-\sum_{i=1}^k\mathbbm{1}\{y_i = 1\}F_{Z_i}(\mathbf{0}),
$$
where $F$ is the cumulative distribution function, \mbox{$Z_i\sim\mathcal{N}\left(\boldsymbol\mu_i, \Sigma_i\right)$} is a $(k{-}1)$-variate Gaussian distribution with 
$$
\mu_{i,j} = 
\begin{cases}
      \pbf\cdot(\fbf_{:,j}(\mathbf{x})-\fbf_{:,i}(\mathbf{x})) & \hspace{-2mm}\text{if } ~ j\in\{1,\dots,i-1\},\\
      \pbf\cdot(\fbf_{:,j+1}(\mathbf{x})-\fbf_{:,i}(\mathbf{x})) & \hspace{-2mm}\text{if } ~ j\in\{i,\dots,k-1\},
    \end{cases}
$$
$$
\hat{\Sigma}_{i, j, k} = \begin{cases}
      (\fbf_{:,j}(\mathbf{x}) - \fbf_{:,i}(\mathbf{x}))\cdot(\fbf_{:,k}(\mathbf{x}) - \fbf_{:,i}(\mathbf{x})) &\hspace{-2mm} \text{if } j < i, k < i,\\
      (\fbf_{:,j+1}(\mathbf{x}) - \fbf_{:,i}(\mathbf{x}))\cdot(\fbf_{:,k}(\mathbf{x}) - \fbf_{:,i}(\mathbf{x})) &\hspace{-2mm} \text{if } j \geq i, k < i,\\
      (\fbf_{:,j}(\mathbf{x}) - \fbf_{:,i}(\mathbf{x}))\cdot(\fbf_{:,k+1}(\mathbf{x}) - \fbf_{:,i}(\mathbf{x})) &\hspace{-2mm} \text{if } j < i, k \geq i,\\
      (\fbf_{:,j+1}(\mathbf{x}) - \fbf_{:,i}(\mathbf{x}))\cdot(\fbf_{:,k+1}(\mathbf{x}) - \fbf_{:,i}(\mathbf{x})) &\hspace{-2mm} \text{if } j \geq i, k \geq i.\\
\end{cases}
$$
\end{restatable}
Though we were not able to obtain an analytic expression for $\CondZero_{\Dcal}^{\Ncal(\mathbf{p}, \I)}(\hbf_{\mathbf{p}})$, $\CondOne_{\Dcal}^{\Ncal(\mathbf{p}, \I)}(\hbf_{\mathbf{p}})$ enabling lower-bounding analogous to \cref{prop:partition_categorical,prop:partition_dirichlet,prop:partition_binary_gaussian,prop:max_binary_gaussian}, this result leads to a proper learning objective that several benchmarks can leverage, as shown in the following section.

\section{Bound optimization}
\label{section:optimization}
By deterministically lower-bounding $\CondZero_{\Dcal}^Q(\hbf)$ and $\CondOne_{\Dcal}^Q(\hbf)$ (\cref{prop:partition_categorical,prop:partition_dirichlet,prop:partition_binary_gaussian,prop:max_binary_gaussian}), we are able to leverage \cref{cor:triple}. We call this method for bounding $\risk_\Dcal(\hbf)$ the \textit{partition bound}.

While $\risk_\Dcal(Q)$ can be estimated with any classical PAC-Bayes bound \citep[e.g.,][]{DBLP:journals/jmlr/Seeger02}, $\CondOne_{\Dcal}^Q(\hbf)$ is bounded with \cref{prop:partition_categorical} when $Q$ is Categorical, \cref{prop:partition_dirichlet} when $Q$ is a Dirichlet distribution, and \cref{prop:partition_binary_gaussian} when $Q$ is Gaussian and $k = 2$. The worst-case analysis for $\CondZero_{\Dcal}^Q(\hbf)$ is done using \cref{prop:max_binary_gaussian} when $Q$ is Gaussian and $k = 2$, and manually otherwise.

This complete greedy algorithm solves an integer linear program to find an optimal solution. We use an implementation provided by the \href{https://github.com/coin-or/prtpy/blob/main}{prtpy} Python package\footnote{\url{https://github.com/coin-or/prtpy/blob/main/prtpy/partitioning/integer_programming.py}}. Interestingly, the partitioned values (the weights) need no special treatment such as rounding: the count of unique values is stored as integers, not the values themselves, so that the algorithm is compatible with our use case. The memory required by the algorithm is not affected by the smallest difference possible between two different values, as it is the case, for example, for dynamic programming algorithms.

\paragraph{Training and post-training heuristics.} During the training phase, we simply optimize the upper bound on the Gibbs risk $\risk_{\Dcal}(Q)$ given by \cref{theo:classical_kl}. Then, we apply several heuristics to tighten the partition bound. This separation is necessary because the deterministic bounding of $\CondZero_{\Dcal}^{Q}(\hbf)$ and $\CondOne_{\Dcal}^{Q}(\hbf)$ is non-differentiable.

\cref{prop:partition_categorical,prop:partition_dirichlet,prop:partition_binary_gaussian,prop:max_binary_gaussian} tells us that the bigger $\tilde{p}=\big|\sum_{p\in\mathbf{p}_1}p - \sum_{p\in\mathbf{p}_2}p\big|$, given that $\mathbf{p}_1$ and $\mathbf{p}_2$ are the result of the partition problem applied to $\mathbf{p}$, the tighter the bound. Thus, if no two partitions of $\pbf$ have similar total values, the better the partition bound. Leveraging this knowledge, after the optimization of $\risk_{\Dcal}(Q)$, we apply three heuristics to improve the partition bound, each applied iteratively until no further improvement is achieved:

\begin{heuristic}\label{heur:1}
We clip the smallest absolute values of $\pbf$ to~$0$, increasing $\tilde{p}$, until the upper bound from \cref{cor:triple} is no longer lowered.
\end{heuristic}%
\begin{heuristic}\label{heur:2}
We apply a coordinate descent on the posterior values (increasing the largest components of~$\pbf$ and decreasing the smallest ones), directly increasing $\tilde{p}$, until the upper bound from \cref{cor:triple} does not lower anymore.
\end{heuristic}%
\begin{heuristic}\label{heur:3}
Since $\tilde{p}$ grows linearly with $||\pbf||_1$, we rescale the L1-norm $||\pbf||_1$ when it allows to reduce the partition bound value.
\end{heuristic}

\section{Related works}
\label{section:related}
\subsection{Existing derandomized approaches via the Bayes classifier}

We now present the main strategy from the PAC-Bayes literature for bounding the deterministic risk. We first need to define the so-called Bayes risk\footnote{The PAC-Bayes definition of the Bayes risk, which we adopt, is different from the risk of the Bayes-optimal predictor from the Bayesian literature.}; that is, the risk of the Bayes classifier:
$$
\Bcal_{\Dcal}(Q) = \underset{(\mathbf{x}, \ybf)\sim\mathcal{D}}{\mathbb{E}}\left[\ell\left(\underset{\hbf'\sim Q}{\mathbb{E}}\hbf',\mathbf{x}, \ybf\right)\right].
$$
The approaches described below rely on bounding the Bayes risk by some function $f$:
$$
\Bcal_{\Dcal}(Q) \leq f\left(\dots\right),
$$
where the function $f$ may or may not involve the Gibbs risk. These approaches for bounding the risk of a single classifier rely on derandomization bounds for the Bayes classifier, which are consistent with the framework we developed. Our approach is more general and practical, in that it allows any hypothesis $\hbf$'s risk to be upper-bounded, not only the Bayes classifier's, which usually implies restricting to models with very specific structure, such as in \cite{DBLP:conf/nips/LetarteGGL19,DBLP:journals/entropy/BiggsG21}.

While many papers discuss ways to bound the Bayes risk \citep{DBLP:conf/nips/LangfordS02,
DBLP:conf/pkdd/LacasseLMT10, 
DBLP:conf/icml/RoyLM11, 
DBLP:journals/jmlr/GermainLLMR15, 
DBLP:journals/ijon/LavioletteMRR17, 
DBLP:conf/nips/MasegosaLIS20,
DBLP:conf/nips/ZantedeschiVMEH21,
DBLP:conf/nips/WuMLIS21, 
DBLP:conf/pkdd/ViallardGHM21}, we focus on four of the most influential ones, using
$$
\Ical_{Q}(\xbf,\ybf) = \underset{\hbf\sim Q}{\mathbb{E}}\ell_{01}\left(\hbf,\xbf, \ybf\right).
$$

\paragraph{First-order bound.} The so-called \textit{factor-2} bound \citep{DBLP:conf/nips/LangfordS02} might be the most classical one. Using Markov's inequality, we obtain
$$
\Bcal_{\Dcal}(Q)\leq 2\underset{(\xbf, \ybf)\sim \Dcal}{\mathbb{E}}\Ical_{Q}(\xbf,\ybf) = 2\risk_{\mathcal{D}}(Q).
$$
The bound considers the individual performance of the hypotheses independently, ignoring their correlation.

\paragraph{Second-order bound.} To address this, \cite{DBLP:conf/nips/MasegosaLIS20} focuses on improving the bounds by accounting for voter correlations, i.e., considering the agreement and/or disagreement of two random voters:
$$
\Bcal_{\Dcal}(Q) \leq 4~\underset{(\xbf, \ybf)\sim \Dcal}{\mathbb{E}}\Ical_{Q}^2(\xbf,\ybf).
$$

\paragraph{Binomial bound.} A generalization of the first-order approach was proposed in \cite{DBLP:journals/jmlr/Shawe-TaylorH09, DBLP:conf/pkdd/LacasseLMT10}, where the Bayes risk is estimated by drawing multiple ($N$) hypotheses and computing the probability that at least half of them make an error:
$$
\Bcal_{\Dcal}(Q) \leq  2\underset{(\xbf, \ybf)\sim \Dcal}{\mathbb{E}}\left[\sum_{j=\frac{N}{2}}^N \binom{N}{k}\Ical_{Q}^{j}(\xbf,\ybf)\left(1-\Ical_{Q}(\xbf,\ybf)\right)^{N-j}\right].
$$

\paragraph{C-bound.} Proposed by \cite{DBLP:conf/nips/LacasseLMGU06}, this bound is derived by considering the joint error and disagreement between two base predictors. Using Chebyshev-Cantelli’s inequality:
$$
\Bcal_{\Dcal}(Q) \leq \frac{\underset{(\xbf, \ybf)\sim \Dcal}{\mathbb{E}}\left[\Ical_{Q}^2(\xbf,\ybf)\right] - \underset{\hbf'\sim Q}{\mathbb{E}}\left[\risk_{\Dcal}^2(\hbf')\right]}{\underset{(\xbf, \ybf)\sim \Dcal}{\mathbb{E}}\left[\Ical_{Q}^2(\xbf,\ybf) - \Ical_{Q}(\xbf,\ybf)\right] + \frac{1}{4}}.
$$

These methods rely on estimating any $\mathbb{E}_{(\xbf, \ybf)\sim \Dcal}\Ical^j_{Q}(\xbf,\ybf)$, for some $j$, with PAC-Bayes bounds. Optimizing the C-bound is a research topic of its own \citep{DBLP:conf/pkdd/ViallardGHM21}.

\subsection{Other derandomization approaches from the literature}

A notable derivation of the PAC-Bayesian framework is found in disintegrated PAC-Bayes bounds \citep{DBLP:conf/colt/BlanchardF07,DBLP:conf/nips/RivasplataKSS20,viallard2024general,DBLP:conf/alt/ClericoFDGD25}, which bound the risk of any classifier \textit{randomly} drawn according to the posterior distribution $Q$. Such bounds apply to the generalization risk of a single classifier, yet this classifier is randomly drawn from the posterior distribution rather than arbitrarily chosen.

Other methods for derandomization involve particular loss functions (losses involving a \textit{margin} parameter \citep{DBLP:conf/aistats/BiggsG22,DBLP:journals/corr/abs-2002-09956}, or sub-Gaussian losses \citep{DBLP:journals/corr/abs-2002-09956}) or hypothesis architecture \citep{DBLP:conf/iclr/NeyshaburBS18,DBLP:conf/iclr/NagarajanK19}.

\subsection{A baseline based on the VC-dimension}

Finally, we consider an alternative paradigm for bounding the deterministic risk of hypotheses that does not involve PAC-Bayes: the VC-dimension (VC-dim) \citep{DBLP:books/daglib/0026015}, which is a measure of the expressive power of hypothesis sets. For any hypothesis set of finite VC-dim $v$, \cite{DBLP:books/daglib/0026015} showed that given $\delta\in[0,1]$, we have, with probability at least $1-\delta$,
$$
\risk_{\Dcal}(\hbf) \leq \risk_{S}(\hbf) + \sqrt{\frac{v\left(\ln\left(\frac{2m}{v}\right)+1\right)+\ln\left(\frac{4}{\delta}\right)}{m}}\,.
$$
It is well-known \citep{DBLP:books/daglib/0026015} that for the class of weighted majority votes involving $n$ fixed base classifiers, i.e., the class of homogeneous halfspaces in~$n$ dimensions, the corresponding VC-dim is also $n$.

\subsection{Contributions contextualization}

\renewcommand{\arraystretch}{1.20}
\setlength{\tabcolsep}{2pt}
\begin{table}[t]
    \centering
    \footnotesize
    \begin{tabular}{|c|c|c|c|c|c|c|}
        \hline
        $k$ & Distribution & First Order & Second Order & Binomial Bound & C-Bound & \textbf{Partition Bound}\\
        \hline
        \multirow{3}{*}{$k=2$} & Categorical & \checkmark \cellcolor{lime} & \checkmark \cellcolor{lime} & \checkmark \cellcolor{lime} & \checkmark \cellcolor{lime} & $\star$ \cellcolor{ForestGreen}\\
        & Dirichlet & $\times$ \cellcolor{red} & $\times$ \cellcolor{red} & $\times$ \cellcolor{red} & $\times$ \cellcolor{red} & $\star$ \cellcolor{ForestGreen}\\
        & Gaussian & \checkmark \cellcolor{lime} & \checkmark \cellcolor{lime} & \checkmark \cellcolor{lime} & $\times$ \cellcolor{red} & $\star$ \cellcolor{ForestGreen}\\
        \cline{1-7}
        \multirow{3}{*}{$k>2$} & Categorical & \checkmark \cellcolor{lime} & \checkmark \cellcolor{lime} & \checkmark \cellcolor{lime} & $\times$ \cellcolor{red} & $\star$ \cellcolor{ForestGreen}\\
        & Dirichlet & $\times$ \cellcolor{red} & $\times$ \cellcolor{red} & $\times$ \cellcolor{red} & $\times$ \cellcolor{red} & $\star$ \cellcolor{ForestGreen}\\
        & Gaussian & $\star$ \cellcolor{ForestGreen} & $\star$ \cellcolor{ForestGreen} & $\star$ \cellcolor{ForestGreen} & $\times$ \cellcolor{red} & $\times$ \cellcolor{red}\\
        \hline
    \end{tabular}
    \caption{Applicability of the true risk bounding method to derandomized majority vote, for binary \mbox{($k = 2$)} and multi-class \mbox{($k > 2$)} classification problems. \textcolor{ForestGreen}{\bf Green} ($\star$): applicable, contribution; \textcolor{lime}{\bf lime} ($\checkmark$): applicable, literature result; \textcolor{red}{\bf red} ($\times$): non-applicable. 
    }
    \label{tab:contributions}
\end{table}

We show in \cref{tab:contributions} the applicability of the various ways to bound deterministic risks using stochastic risk found in the literature, and regarding our proposed approaches. We emphasize that no method from the literature permits the use of the Dirichlet distribution, and that we provide an objective for many baselines to use the Gaussian distribution for multi-class tasks. Also, while the methods from the literature rely on the simultaneous correctness of the predictions from randomly drawn hypotheses ($\mathbb{E}_{(\xbf, \ybf)\sim \Dcal}\Ical^j_{Q}(\xbf,\ybf)$), our method involves either agreement or disagreement with the hypothesis of interest ($\hbf$) through a conditional expectation, which is another particularity of our approach.

\section{Numerical experiments}
\label{section:experiments}
We compare our proposed approach for bounding deterministic risk, the partition bound (Part), to all of the baselines discussed in Section \ref{section:related}: First-order, Second-order, Binomial, the C-Bound, and finally, the VC-dim bound. We experiment in two different settings\footnote{The code to reproduce these experiments will be made public upon acceptance.}.

In the first setting, we consider binary classification tasks and majority votes of data-independent hypotheses consisting of axis-aligned decision stumps, with thresholds evenly spread over the input space (10 per feature). The PAC-Bayes bound used for bounding $\mathbb{E}_{(\xbf, \ybf)\sim \Dcal}\Ical^j_{Q}(\xbf,\ybf)$ is Seeger's bound \citep{DBLP:journals/jmlr/Seeger02} with Maurer's trick \citep{DBLP:journals/corr/cs-LG-0411099} (see Appendix \ref{appendix:data-dep-independent}).

In the second setting, we consider multi-class classification tasks and majority votes of data-dependent hypotheses, consisting of a Random Forest \citep{DBLP:journals/ml/Breiman01} of 200 trees as a set of voters $T$, sampling $\sqrt{d}$ random features to ensure voter diversity, optimizing the Gini impurity score without bounding their maximal depth. To obtain a PAC-Bayes bound that allows learning the hypotheses on the training data, we rely on the following scheme: we split the training set $S$ into two sets $S_1$ and $S_2$; we learn half the trees $T_1$ of the random forest on $S_1$, the other half $T_2$ on $S_2$ so that $T_1\cup T_2 = T$; we use a PAC-Bayes bound from \cite{DBLP:conf/nips/ZantedeschiVMEH21} to bound the expected risk of the resulting stochastic majority vote, requiring that we evaluate $T_1$ on $S_2$, $T_2$ on $S_1$. Simply put: $S_1$ is used to learn $T_1$ and a posterior distribution over $T_2$ (since $T_2$ is independent of $S_1$), which is in turn used in the bound evaluation, and vice versa. For the VC-dim-based approach, we simply reserved half the data for the learning of the forests and the other half for the majority vote weighting. See Section \ref{appendix:data-dep-independent} for more details.

We experimented, for every baseline and our approaches, with every permissible distribution (see \cref{tab:contributions}) and reported the results given by the best bound attained by a single distribution. We consider several classification datasets from various sources (see \cref{appendix:experimental} for a complete description of the various datasets). In the first setting, we considered the following tasks: ADULT, CODRNA, HABER, MUSH, PHIS, SVMG, TTT; in the second, we considered the following: FASHION, MNIST, PEND, PROTEIN, SENSOR.

We train the models by Stochastic Gradient Descent (SGD) using Adam \citep{DBLP:journals/corr/KingmaB14}. Though a generalization of the C-Bound baseline to multi-class tasks has been proposed \citep{DBLP:nips/multi-cbound}, no learning algorithm is provided: we only use this baseline on binary-class tasks using the learning scheme of \citet[Algorithm 3]{DBLP:conf/pkdd/ViallardGHM21}. The objective function for our proposed approach, as well as the PAC-Bayesian baselines, corresponds to their bounds. Therefore, for these methods, the training pipeline is the same; only the objective (bound) and the number of hyperparameters to tune are different. No validation set is thus required for these methods. On the other hand, the VC-dim-based bound has its training cross-entropy loss as its learning objective and thus requires using a validation set for the hyperparameter selection. See \cref{appendix:experimental} for more details on the experiments.

\subsection{Discussion}

\setlength{\tabcolsep}{1.75pt}
\renewcommand{\arraystretch}{1.05}
\begin{table*}
    \caption{Comparison of our proposed approach and a few baselines for different types of predictors. \textbf{Bolded} and \underline{underlined} values are respectively within one standard deviation of the best bound and test error values (\%). Average and standard deviations computed over 5 random seeds.}
    \centering
    \resizebox{\textwidth}{!}{%
    \begin{tabular}{|c||cc|cc|cc|cc|cc||cc|}
        \hline
        \multirow{2}{*}{Task} & \multicolumn{2}{c|}{First Order} & \multicolumn{2}{c|}{Second Order} & \multicolumn{2}{c|}{Binomial Bound} & \multicolumn{2}{c|}{C-Bound} & \multicolumn{2}{c||}{VC-dimension} & \multicolumn{2}{c|}{\textbf{Partition Bound}}\\
         & Bound & Test error & Bound & Test error & Bound & Test error & Bound & Test error & Bound & Test error & Bound & Test error \\
\hline
ADULT & 46.9 $\pm$ 0.0 & 22.2 $\pm$ 0.1 & 52.9 $\pm$ 0.1 & \underline{15.8} $\pm$ 0.2 & 51.4 $\pm$ 0.0 & 24.1 $\pm$ 0.0 & 74.5 $\pm$ 0.4 & 21.0 $\pm$ 1.6 & 78.9 $\pm$ 0.0 & 17.2 $\pm$ 0.3 & \textbf{23.2} $\pm$ 2.4 & 21.7 $\pm$ 2.5 \\ 
CODRNA & 40.2 $\pm$ 0.1 & 12.0 $\pm$ 0.1 & 47.4 $\pm$ 0.1 & \underline{11.7} $\pm$ 0.1 & 42.8 $\pm$ 5.2 & 15.2 $\pm$ 4.5 & 63.7 $\pm$ 1.0 & 23.1 $\pm$ 1.3 & 38.8 $\pm$ 1.6 & 23.1 $\pm$ 1.5 & \textbf{24.7} $\pm$ 0.3 & 23.6 $\pm$ 0.3 \\ 
HABER & 75.2 $\pm$ 0.9 & \underline{26.5} $\pm$ 1.8 & 110.5 $\pm$ 1.1 & \underline{26.1} $\pm$ 2.1 & 84.7 $\pm$ 0.8 & \underline{26.1} $\pm$ 0.7 & 100.0 $\pm$ 0.0 & \underline{27.1} $\pm$ 1.8 & 112.0 $\pm$ 1.0 & \underline{28.4} $\pm$ 1.8 & \textbf{37.8} $\pm$ 0.5 & \underline{26.5} $\pm$ 1.8 \\ 
MUSH & 12.2 $\pm$ 0.3 & 4.8 $\pm$ 0.5 & 20.6 $\pm$ 1.2 & \underline{0.5} $\pm$ 0.3 & 12.6 $\pm$ 0.4 & 1.1 $\pm$ 0.1 & 22.0 $\pm$ 0.5 & 4.8 $\pm$ 0.5 & 59.1 $\pm$ 0.1 & 4.8 $\pm$ 0.5 & \textbf{6.1} $\pm$ 0.1 & 4.8 $\pm$ 0.5 \\ 
PHIS & 25.8 $\pm$ 0.3 & 11.1 $\pm$ 0.6 & 33.6 $\pm$ 0.2 & \underline{6.4} $\pm$ 0.5 & 27.5 $\pm$ 0.2 & \underline{6.8} $\pm$ 0.5 & 37.1 $\pm$ 0.5 & 10.4 $\pm$ 1.3 & 85.2 $\pm$ 0.2 & 11.1 $\pm$ 0.6 & \textbf{24.9} $\pm$ 1.0 & 11.1 $\pm$ 0.6 \\ 
SVMG & 17.4 $\pm$ 0.6 & \underline{7.0} $\pm$ 1.0 & 28.2 $\pm$ 0.7 & \underline{7.0} $\pm$ 1.0 & 22.0 $\pm$ 0.6 & \underline{7.0} $\pm$ 1.0 & 29.2 $\pm$ 0.7 & \underline{7.0} $\pm$ 1.0 & 37.6 $\pm$ 0.3 & \underline{7.0} $\pm$ 1.0 & \textbf{8.7} $\pm$ 0.3 & \underline{7.0} $\pm$ 1.0 \\ 
TTT & 75.9 $\pm$ 1.9 & \underline{30.4} $\pm$ 3.8 & 100.0 $\pm$ 0.7 & \underline{30.2} $\pm$ 2.7 & 87.2 $\pm$ 1.2 & \underline{29.5} $\pm$ 4.3 & 94.8 $\pm$ 0.7 & \underline{30.4} $\pm$ 3.8 & 117.0 $\pm$ 1.8 & \underline{32.2} $\pm$ 2.9 & \textbf{48.4} $\pm$ 12.6 & \underline{30.4} $\pm$ 3.8 \\ 
\hline
FASHION & 40.7 $\pm$ 0.2 & 17.1 $\pm$ 0.2 & 58.6 $\pm$ 0.2 & \underline{13.0} $\pm$ 0.3 & 49.9 $\pm$ 0.3 & \underline{13.0} $\pm$ 0.3 & N/A & N/A & 40.3 $\pm$ 0.3 & 18.3 $\pm$ 0.3 & \textbf{26.6} $\pm$ 0.1 & 25.4 $\pm$ 0.2 \\ 
MNIST & 26.9 $\pm$ 0.1 & 8.0 $\pm$ 0.4 & 44.7 $\pm$ 0.1 & \underline{4.4} $\pm$ 0.1 & 35.8 $\pm$ 0.2 & \underline{4.3} $\pm$ 0.2 & N/A & N/A & 37.4 $\pm$ 2.6 & 15.8 $\pm$ 2.7 & \textbf{23.6} $\pm$ 0.2 & 22.6 $\pm$ 0.3 \\ 
PEND & 12.0 $\pm$ 0.2 & 2.5 $\pm$ 0.3 & 17.1 $\pm$ 0.1 & \underline{1.5} $\pm$ 0.2 & 15.5 $\pm$ 0.2 & \underline{1.5} $\pm$ 0.2 & N/A & N/A & 52.7 $\pm$ 0.6 & 5.8 $\pm$ 0.7 & \textbf{8.9} $\pm$ 0.2 & \underline{1.4} $\pm$ 0.2 \\ 
PROTEIN & 84.9 $\pm$ 0.6 & \underline{34.9} $\pm$ 0.3 & 138.3 $\pm$ 0.3 & 36.6 $\pm$ 0.5 & 120.1 $\pm$ 1.3 & 55.0 $\pm$ 0.5 & N/A & N/A & 73.4 $\pm$ 0.3 & 39.7 $\pm$ 0.5 & \textbf{55.8} $\pm$ 0.4 & 53.8 $\pm$ 0.6 \\ 
SENSOR & 16.8 $\pm$ 0.4 & 4.1 $\pm$ 0.3 & 21.2 $\pm$ 1.0 & 3.9 $\pm$ 0.4 & 15.3 $\pm$ 1.0 & 3.0 $\pm$ 0.2 & N/A & N/A & 32.9 $\pm$ 2.0 & 9.6 $\pm$ 2.0 & \textbf{1.4} $\pm$ 0.1 & \underline{0.4} $\pm$ 0.2 \\ 
\hline
    \end{tabular}}
    \label{tab:results}
\end{table*}

\setlength{\tabcolsep}{1.75pt}
\renewcommand{\arraystretch}{0.95}
\begin{table}[t]
    \caption{Proposed approach's bound and test error (\%) for every distribution.  \textbf{Bolded} and \underline{underlined} values are respectively within one standard deviation of the best bound and test error values. Average and standard deviations computed over 5 random seeds.}
    \centering
    \resizebox{0.68\textwidth}{!}{%
    \begin{tabular}{|c||cc|cc|cc|}
        \hline
        \multirow{2}{*}{Task} & \multicolumn{2}{c|}{Categorical Distribution} & \multicolumn{2}{c|}{Dirichlet Distribution} & \multicolumn{2}{c|}{Gaussian  Distribution}\\
         & Partition Bound & Test error & Partition Bound & Test error & Partition Bound & Test error \\
        \hline
ADULT & 46.4 $\pm$ 0.4 & 22.2 $\pm$ 0.1 & \textbf{23.2} $\pm$ 2.4 & \underline{21.7} $\pm$ 2.5 & 51.0 $\pm$ 0.0 & 24.1 $\pm$ 0.0 \\ 
CODRNA & \textbf{24.7} $\pm$ 0.3 & 23.6 $\pm$ 0.3 & 29.4 $\pm$ 4.8 & 28.4 $\pm$ 4.6 & 39.1 $\pm$ 1.4 & \underline{17.1} $\pm$ 7.2 \\ 
HABER & \textbf{37.8} $\pm$ 0.5 & \underline{26.5} $\pm$ 1.8 & 73.6 $\pm$ 7.9 & 34.2 $\pm$ 7.2 & 71.7 $\pm$ 4.2 & \underline{25.8} $\pm$ 3.0 \\ 
MUSH & \textbf{6.1} $\pm$ 0.1 & 4.8 $\pm$ 0.5 & 8.0 $\pm$ 1.0 & \underline{4.1} $\pm$ 1.8 & 19.0 $\pm$ 0.5 & \underline{2.6} $\pm$ 0.6 \\ 
PHIS & \textbf{24.9} $\pm$ 1.0 & 11.1 $\pm$ 0.6 & 26.8 $\pm$ 0.2 & \underline{6.7} $\pm$ 0.4 & 27.5 $\pm$ 0.2 & \underline{6.6} $\pm$ 0.5 \\ 
SVMG & \textbf{8.7} $\pm$ 0.3 & 7.0 $\pm$ 1.0 & \textbf{9.5} $\pm$ 1.7 & \underline{5.3} $\pm$ 1.4 & 14.0 $\pm$ 0.2 & 7.0 $\pm$ 1.0 \\ 
TTT & \textbf{48.4} $\pm$ 12.6 & \underline{30.4} $\pm$ 3.8 & 84.3 $\pm$ 1.3 & \underline{29.9} $\pm$ 4.1 & 82.7 $\pm$ 9.5 & \underline{36.5} $\pm$ 6.7 \\ 
\hline
FASHION & \textbf{26.6} $\pm$ 0.1 & 25.4 $\pm$ 0.2 & 44.8 $\pm$ 0.3 & \underline{13.0} $\pm$ 0.3 & N/A & N/A \\ 
MNIST & \textbf{23.6} $\pm$ 0.2 & 22.6 $\pm$ 0.3 & 31.3 $\pm$ 0.2 & \underline{4.3} $\pm$ 0.2 & N/A & N/A \\ 
PEND & 10.6 $\pm$ 0.2 & 8.3 $\pm$ 0.7 & \textbf{8.9} $\pm$ 0.2 & \underline{1.4} $\pm$ 0.2 & N/A & N/A \\ 
PROTEIN & \textbf{55.8} $\pm$ 0.4 & \underline{53.8} $\pm$ 0.6 & \textbf{55.8} $\pm$ 3.3 & 54.9 $\pm$ 0.5 & N/A & N/A \\ 
SENSOR & 3.4 $\pm$ 0.3 & 2.9 $\pm$ 0.3 & \textbf{1.4} $\pm$ 0.1 & \underline{0.4} $\pm$ 0.2 & N/A & N/A \\ 
\hline
    \end{tabular}}
    \label{tab:results_distributions}
\end{table}

\cref{tab:results} shows that, on a total of 12 datasets in the two different learning paradigms, the proposed approach's generalization bound is always at least as good as the baselines'. The VC-dim-based approach provides loose generalization bounds because the bound is penalized proportionally to the base classifier set size, even when some of these classifiers end up having a small role in the majority vote; the PAC-Bayes-based approaches are more robust to this issue by providing those with a weight similar to that of the prior distribution. In many cases, the proposed bound corresponds to about half the best runner-up bound (ADULT, CODRNA, SVMG, etc.) On five tasks out of twelve, the proposed approach even has a test error at most one standard deviation away from the best test value. Combined, these results show the dominance of our approach in many scenarios.

These results show a tradeoff between excellent accuracy without an adequate certificate, and good performance with a representative certificate. The latter situation is well handled by our proposed approach, thanks to the method's efficiency in generating tight generalization bounds, sometimes at the expense of empirical performance.
The tradeoff also depends on the distribution used, as shown in \cref{tab:results_distributions}, which displays the partition bound and test error obtained for each distribution by our method. The gap between the reported test error and the best baselines is correlated with the selected distribution in our method: in general, the test error when using the Dirichlet distribution is similar to that of the baselines, but in some situations, it is not selected because of the value of its generalization bound (see \cref{appendix:experimental} for a table of the selected distribution for each method and dataset). Therefore, even though the Categorical distribution definitely seems to have an advantage over both the Dirichlet and the Gaussian distributions in terms of the generalization bound, the latter two distributions exhibit low test error results.

\section{Conclusion}
\label{section:conclusion}
In this work, we developed a general framework for deriving, from PAC-Bayes, generalization guarantees holding for any single hypothesis: \textit{Derandomization} bounds. We develop such an oracle bound before specializing it to majority votes with weighting involving either the Categorical, the Dirichlet, or the Gaussian distribution. We test (and compare to several benchmarks) our approach on many binary and multi-class classification tasks, providing empirical evidence of the tightness of our approach. 

Our derandomization framework is consistent with existing PAC-Bayesian bounding methods while being more general: it allows bounding the risk of any single classifier in $\Hcal$, not only that of the Bayes classifier. 

\paragraph{Future work. } Lower bounds on the risk of a single hypothesis cannot be obtained using the methods surveyed in \cref{section:related}, yet they can be obtained through our derandomization framework, using, for example, \cref{cor:triple-single,cor:triple}. Though lower bounds on risk metrics are less popular than upper ones, they appear necessary in bounding metrics involving the risk of a single hypothesis on various tasks, such as in fairness computation \citep{DBLP:journals/ijon/OnetoDPS20}. Therefore, a unique and practical application of our derandomization method could be found in bounding the value of such a metric for single classifiers.

There are several other directions for future work: extending beyond majority votes to neural networks, deriving multi-class Gaussian bounds, and developing end-to-end differentiable approximations to the partition problem. Also, a natural extension of this work lies in using any (possibly unbounded) loss $\ell~:~\Hcal\times\Xcal\times\Ycal\rightarrow\Rbb_{\geq 0}$. We believe this could be achieved by defining $\CondZero_{\Dcal}^Q(\hbf)$ and $\CondOne_{\Dcal}^Q(\hbf)$ as such:
\begin{align*}
\CondZero_{\Dcal}^Q(\hbf) &= \underset{(\mathbf{x}, \ybf)\sim \Dcal}{\mathbb{E}} \left[\underset{\hbf'\sim Q}{\mathbb{E}} \ell(\hbf',\mathbf{x}, \ybf)~\bigg|~\ell(\hbf,\mathbf{x}, \ybf) \leq \gamma\right],\\
\CondOne_{\Dcal}^Q(\hbf) &= \underset{(\mathbf{x}, \ybf)\sim \Dcal}{\mathbb{E}} \left[\underset{\hbf'\sim Q}{\mathbb{E}} \ell(\hbf',\mathbf{x}, \ybf)~\bigg|~\ell(\hbf,\mathbf{x}, \ybf) > \gamma\right],
\end{align*}
for some $\gamma\geq 0$. Adapting the loss function to reflect this $\gamma$ parameter and interpreting it as a \textit{margin}-related parameter could enable utilizing margin-based generalization bounds \citep{DBLP:journals/ai/GaoZ13a,DBLP:conf/nips/ZantedeschiVMEH21,DBLP:conf/nips/BiggsZG22,DBLP:conf/aistats/BiggsG22} as building blocks in our derandomization framework, further tying up our approach with existing PAC-Bayes bounds from the literature.

\bibliographystyle{tmlr}
\bibliography{references}

\appendix

\section{PAC-Bayes bounds from the literature}
\label{appendix:literature}
\label{appendix:data-dep-independent}
\paragraph{Data-independent case. }


\begin{theorem}\label{theo:classical_kl}\textup{\citep{DBLP:journals/jmlr/Seeger02, DBLP:journals/corr/cs-LG-0411099}}
    Let $\kl(q, p) = q \ln(\frac{q}{p}) + (1 - q) \ln(\frac{1-q}{1-p})$. For any distribution $\Dcal$, hypothesis set $\Hcal$, prior distribution $P$ over $\Hcal$, $\delta\in (0,1]$, we have with probability at least $1-\delta$ over the random choice $S\sim \Dcal^m$ that for every $Q$ over $\Hcal$:
$$
\kl\left(\risk_S(Q)\bigg|\bigg|\risk_\Dcal(Q)\right)\leq \frac{1}{m}\left[\KL(Q||P)+\ln\left(\frac{2\sqrt{m}}{\delta}\right)\right].
$$
\end{theorem}

\paragraph{Data-dependent case. }

Here, we present a cross-bounding certificate \citep{DBLP:conf/nips/ZantedeschiVMEH21} that allows us to learn and evaluate the set of base classifiers without held-out data. More precisely, we split the training data $S$ into two subsets ($S_{\leq j} = \{(\xbf_i, \ybf_i)\in S\}_{i=1}^j$ and $S_{> j} = \{(\xbf_i, \ybf_i)\in S\}_{i=j+1}^m$ for some $j$) and we learn a set of base classifiers on each data split independently (determining the hypothesis spaces $\Hcal_{\leq j}$ and $\Hcal_{> j}$). We refer to the prior distribution over $\Hcal_{\leq j}$ as $P_{\leq j}$ and to the prior distribution over $\Hcal_{> j}$ as $P_{> j}$. In the same way, we can then define a posterior distribution per hypothesis space: $Q_{\leq j}$ and $Q_{> j}$. The following theorem shows that we can bound the expected risk of any convex combination of the two posteriors, as long as their empirical risks are evaluated on the data split that was not used for learning their respective priors.

\begin{theorem}\label{theo:data-dependent_kl}\textup{\citep{DBLP:conf/nips/ZantedeschiVMEH21}}
    Let $\kl(q, p) = q \ln(\frac{q}{p}) + (1 - q) \ln(\frac{1-q}{1-p})$. For any distribution $\Dcal$, hypothesis sets $\Hcal_{\leq j}$ and $\Hcal_{> j}$, prior distributions $P_{\leq j}$ over $\Hcal_{\leq j}$ and $P_{> j}$ over $\Hcal_{> j}$, $\delta\in (0,1]$ and $\alpha\in[0,1]$, we have with probability at least $1-\delta$ over $S\sim \Dcal^m$ that for every $Q_{\leq j}$ over $\Hcal_{\leq j}$ and $Q_{> j}$ over $\Hcal_{> j}$:
\begin{align*}
\kl\bigg(\alpha\risk_{S_{>j}}(Q_{\leq j})+(1-\alpha)\risk_{S_{\leq j}}(Q_{>j})~\bigg|\bigg|~&\alpha\risk_{\Dcal}(Q_{\leq j})+(1-\alpha)\risk_{\Dcal}(Q_{>j})\bigg)\\
\leq~&\frac{\alpha\KL(Q_{>j}||P_{>j})}{j}+\frac{(1-\alpha)\KL(Q_{\leq j}||P_{\leq j})}{m-j}+\frac{1}{m}\ln\left(\frac{4\sqrt{j(m-j)}}{\delta}\right).
\end{align*}
\end{theorem}

\section{Mathematical proofs}
\label{appendix:proofs}
\general*

\begin{proof}
\begin{align*}
    \underset{\mathbf{h}' \sim Q}{\Ebb} \risk_{\Dcal}(\mathbf{h}') &= \underset{{(\mathbf{x}, \ybf)\sim \Dcal}}{\mathbb{E}}~\underset{\mathbf{h}'\sim Q}{\mathbb{E}} \ell(\mathbf{h}',\mathbf{x}, \ybf)&\\
    &= \underset{{(\mathbf{x}, \ybf)\sim \Dcal}}{\mathbb{P}}\Big(\ell(\mathbf{h},\mathbf{x}, \ybf) = 1\Big)\underset{{(\mathbf{x}, \ybf)\sim \Dcal}}{\mathbb{E}}\left[\underset{\mathbf{h}'\sim Q}{\mathbb{E}} \ell(\mathbf{h}',\mathbf{x}, \ybf)~\Bigg|~\ell(\mathbf{h},\mathbf{x}, \ybf) = 1\right]+&\\
    &~~~~ \underset{{(\mathbf{x}, \ybf)\sim \Dcal}}{\mathbb{P}}\Big(\ell(\mathbf{h},\mathbf{x}, \ybf) = 0\Big)\underset{{(\mathbf{x}, \ybf)\sim \Dcal}}{\mathbb{E}}\left[\underset{\mathbf{h}'\sim Q}{\mathbb{E}} \ell(\mathbf{h}',\mathbf{x}, \ybf)~\Bigg|~\ell(\mathbf{h},\mathbf{x}, \ybf) = 0\right]&\\
    &= \CondOne_{\Dcal}^Q(\hbf)\underset{{(\mathbf{x}, \ybf)\sim \Dcal}}{\mathbb{P}}\Big(\ell(\mathbf{h},\mathbf{x}, \ybf) = 1\Big)+\CondZero_{\Dcal}^Q(\hbf)\underset{{(\mathbf{x}, \ybf)\sim \Dcal}}{\mathbb{P}}\Big(\ell(\mathbf{h},\mathbf{x}, \ybf) = 0\Big)&\\
    &= \CondOne_{\Dcal}^Q(\hbf)\underset{{(\mathbf{x}, \ybf)\sim \Dcal}}{\mathbb{P}}\Big(\ell(\mathbf{h},\mathbf{x}, \ybf) = 1\Big)+\CondZero_{\Dcal}^Q(\hbf)\Big(1-\underset{{(\mathbf{x}, \ybf)\sim \Dcal}}{\mathbb{P}}\Big(\ell(\mathbf{h},\mathbf{x}, \ybf) = 1\Big)\Big)&\\
    &= (\CondOne_{\Dcal}^Q(\hbf)-\CondZero_{\Dcal}^Q(\hbf))\underset{{(\mathbf{x}, \ybf)\sim \Dcal}}{\mathbb{P}}\Big(\ell(\mathbf{h},\mathbf{x}, \ybf) = 1\Big)+\CondZero_{\Dcal}^Q(\hbf)&\\
    &= (\CondOne_{\Dcal}^Q(\hbf)-\CondZero_{\Dcal}^Q(\hbf))\underset{(\mathbf{x}, \ybf)\sim \Dcal}{\mathbb{E}} \ell(\mathbf{h},\mathbf{x}, \ybf)+\CondZero_{\Dcal}^Q(\hbf)&\\
    &= (\CondOne_{\Dcal}^Q(\hbf)-\CondZero_{\Dcal}^Q(\hbf))\risk_{\Dcal}(\hbf)+\CondZero_{\Dcal}^Q(\hbf)\,.&
\end{align*}
The second equality in the proof of \cref{prop:general} requires that $\mathbb{P}_{{(\mathbf{x}, \ybf)\sim \Dcal}}\Big(\ell(\hbf,\mathbf{x}, \ybf) = \frac{1\pm 1}{2}\Big)$ have non-zero probability; if it is not the case, we assign a value of $0$ to the corresponding term $\mathbb{E}_{{(\mathbf{x}, \ybf)\sim \Dcal}}\left[\mathbb{E}_{\hbf'\sim Q} \ell(\hbf',\mathbf{x}, \ybf)~\Big|~\ell(\hbf,\mathbf{x}, \ybf) = \frac{1\pm1}{2}\right]$. The main result is obtained by a simple rearrangement of the terms, given $\CondOne_{\Dcal}^Q(\hbf) - \CondZero_{\Dcal}^Q(\hbf) \neq 0$.
\end{proof}

\begin{lemma}\textup{\citep{DBLP:journals/corr/cs-LG-0411099}}\label{lem:exp_to_sum}
    For a distribution $\Dcal$ over $\Xcal\times\Ycal$, a loss function $\ell:\Hcal\times\Xcal\times\Ycal\rightarrow[0,1]$, a hypothesis $\hbf\in\Hcal$, and a convex function $\kappa:[0,1]\rightarrow\Rbb$:
$$
\underset{S\sim\Dcal^{m}}{\Ebb} \kappa\left(\risk_S(\hbf)\right) \leq \sum_{i=0}^m \binom{m}{i} \left(\risk_{\Dcal}(\hbf)\right)^i\left(1-\risk_{\Dcal}(\hbf)\right)^{m-i}\kappa\left(\frac{i}{m}\right).
$$
\end{lemma}

\conditional*

\begin{proof}
Let us prove the first statement of the main result; the second one can be proved using similar manipulations. Given a posterior $Q$, let $\Dcal|\{\ell(\hbf,\mathbf{x}, \ybf) = 0\} := \Dcal_{(\hbf,0)}$ be a distribution with probability density function 
$$
f_{\Dcal_{(\hbf,0)}}(\xbf,\ybf) \ = \  \frac{f_{\Dcal}(\xbf,\ybf)\cdot \mathbbm{1}\{\ell(\hbf,\xbf,\ybf)=0)\}}{\Pbb_\Dcal\big(\ell(\hbf,\xbf,\ybf)=0\big)}\,.
$$
Note that $f_{\Dcal}(\xbf,\ybf)$ can then be written as
$$
f_{\Dcal}(\xbf,\ybf) = f_{\Dcal_{(\hbf,0)}}(\xbf,\ybf)\cdot\Pbb_\Dcal\big(\ell(\hbf,\xbf,\ybf)=0\big) + f_{\Dcal_{(\hbf,1)}}(\xbf,\ybf)\cdot\Pbb_\Dcal\big(\ell(\hbf,\xbf,\ybf)=1\big).
$$
Thus, every example from $\Dcal$ can be seen as either a draw from $\Dcal_{(\hbf,0)}$, or a draw from $\Dcal_{(\hbf,1)}$, depending on where it lands in $\Xcal\times\Ycal$. Thus, for any dataset $S\sim\Dcal^m$ and hbfothesis $\hbf$ independent of $S$, any example from $S_{(\hbf,0)}$ can be seen as a realization from~$\Dcal_{(\hbf,0)}$.

We want to upper bound $\kl\left(\CondZero^Q_{S}(\hbf)\Big|\Big|\CondZero^Q_{\Dcal}(\hbf)\right)$ for every posterior $Q$. Given a dataset $S$, we have
\begin{align*}
    \forall Q \textup{ over } \Hcal~:~\qquad &~|S_{(\hbf,0)}|\cdot\kl\left(\CondZero^Q_{S}(\hbf)\Big|\Big|\CondZero^Q_{\Dcal}(\hbf)\right)\\
    =&~|S_{(\hbf,0)}|\cdot\kl\left(\underset{\hbf'\sim Q}{\mathbb{E}}\risk_{S_{(\hbf,0)}}(\hbf')\Big|\Big|\underset{\hbf'\sim Q}{\mathbb{E}}\risk_{\Dcal_{(\hbf,0)}}(\hbf')\right)\\
    \leq&~|S_{(\hbf,0)}|\underset{\hbf'\sim Q}{\mathbb{E}}\kl\left(\risk_{S_{(\hbf,0)}}(\hbf')\Big|\Big|\risk_{\Dcal_{(\hbf,0)}}(\hbf')\right)~~\langle\textup{Jenson's inequality (kl is convex)}\rangle\\
    \leq&~\KL(Q||P)+\ln\left(\underset{\hbf'\sim P}{\mathbb{E}}e^{|S_{(\hbf,0)}|\cdot\kl\left(\risk_{S_{(\hbf,0)}}(\hbf')\Big|\Big|\risk_{\Dcal_{(\hbf,0)}}(\hbf')\right)}\right).~~\langle\textup{Change of measure}\rangle
\end{align*}
Let's now consider the random variable, with respect to $S$: 
$$X_{\hbf, \Hcal, S, P} = \mathbb{E}_{\hbf'\sim P}\exp\left(|S_\hbf|\cdot\kl\left(\risk_{S_{(\hbf,0)}}(\hbf')\Big|\Big|\risk_{\Dcal_{(\hbf,0)}}(\hbf')\right)\right).$$
With Markov's inequality, we have 
$$
\underset{S\sim\Dcal^{m}}{\Pbb}\left(X_{\hbf, \Hcal, S, P}\leq \frac{1}{\delta} \underset{S'\sim\Dcal^{m}}{\Ebb}X_{\hbf, \Hcal, S', P}\right) \geq 1-\delta.
$$
Thus, with probability at least $1-\delta$ over $S\sim\Dcal^{m}$:
$$
\forall Q \textup{ over } \Hcal~:~|S_{(\hbf,0)}|\cdot\kl\left(\CondZero^Q_{S}(\hbf)\Big|\Big|\CondZero^Q_{\Dcal}(\hbf)\right)\leq \KL(Q||P)+\ln\left(\frac{1}{\delta} \underset{S'\sim\Dcal^{m}}{\Ebb}X_{\hbf, \Hcal, S', P}\right).
$$
We bound $\underset{S'\sim\Dcal^{m}}{\Ebb}X_{\hbf, \Hcal, S', P}$ as such:
\begin{align*}
    &~~~~~\underset{S'\sim\Dcal^{m}}{\Ebb}X_{\hbf, \Hcal, S', P}\\
    &= \underset{S'\sim\Dcal^{m}}{\Ebb}~\underset{\hbf'\sim P}{\mathbb{E}}e^{|S'_{(\hbf,0)}|\cdot\kl\left(\risk_{S'_{(\hbf,0)}}(\hbf')\Big|\Big|\risk_{\Dcal_{(\hbf,0)}}(\hbf')\right)}\\
    &= \underset{\hbf'\sim P}{\mathbb{E}}~\underset{S'\sim\Dcal^{m}}{\Ebb}~e^{|S'_{(\hbf,0)}|\cdot\kl\left(\risk_{S'_{(\hbf,0)}}(\hbf')\Big|\Big|\risk_{\Dcal_{(\hbf,0)}}(\hbf')\right)}\\
    &= \underset{\hbf'\sim P}{\mathbb{E}}~\sum_{i=0}^{m}\underset{S'\sim\Dcal^{m}}{\Pbb}\left(|S'_{(\hbf,0)}|=i\right)\underset{S'\sim\Dcal^{m}}{\Ebb}\left[e^{|S'_{(\hbf,0)}|\cdot\kl\left(\risk_{S'_{(\hbf,0)}}(\hbf')\Big|\Big|\risk_{\Dcal_{(\hbf,0)}}(\hbf')\right)}~\Bigg|~|S'_{(\hbf,0)}|=i\right]\\
    &= \underset{\hbf'\sim P}{\mathbb{E}}~\sum_{i=0}^{m}\underset{S'\sim\Dcal^{m}}{\Pbb}\left(|S'_{(\hbf,0)}|=i\right)\underset{\widetilde{S}\sim\Dcal_{(\hbf,0)}^{i}}{\Ebb}e^{i\cdot\kl\left(\risk_{\widetilde{S}}(\hbf')\Big|\Big|\risk_{\Dcal_{(\hbf,0)}}(\hbf')\right)}~~\langle\textup{ Def. of }\Dcal_{(\hbf,0)} \rangle\\
    &\leq \underset{\hbf'\sim P}{\mathbb{E}}~\sum_{i=0}^{m}\underset{S'\sim\Dcal^{m}}{\Pbb}\left(|S'_{(\hbf,0)}|=i\right)~\sum_{j=0}^i \binom{i}{j}\left(\risk_{\Dcal_{(\hbf,0)}}(\hbf')\right)^j\left(1-\risk_{\Dcal_{(\hbf,0)}}(\hbf')\right)^{i-j}e^{i\cdot\kl\left(\frac{j}{i}\Big|\Big|\risk_{\Dcal_{(\hbf,0)}}(\hbf')\right)}~~\langle\textup{\autoref{lem:exp_to_sum}}\rangle\\
    &\leq \underset{\hbf'\sim P}{\mathbb{E}}~\sum_{i=0}^{m}\underset{S'\sim\Dcal^{m}}{\Pbb}\left(|S'_{(\hbf,0)}|=i\right)~~\sup_{r\in[0,1]}\sum_{j=0}^i \binom{i}{j}\left(r\right)^j\left(1-r\right)^{i-j}e^{i\cdot\kl\left(\frac{j}{i}\Big|\Big|r\right)}\\
    &\leq \underset{\hbf'\sim P}{\mathbb{E}}~\sum_{i=0}^{m}\underset{S'\sim\Dcal^{m}}{\Pbb}\left(|S'_{(\hbf,0)}|=i\right)~2\sqrt{i}\qquad\langle\textup{\cite{DBLP:journals/corr/cs-LG-0411099}}\rangle\\
    &\leq \underset{\hbf'\sim P}{\mathbb{E}}~2\sqrt{m}\\
    &=2\sqrt{m}\,.
\end{align*}
Plugging this into our previous result, we obtain
\begin{align*}
\underset{S\sim\Dcal^{m}}{\Pbb}\left(\forall Q\text{ over }\Hcal~:~\kl\left(\CondZero^Q_{S}(\hbf)\Big|\Big|\CondZero^Q_{\Dcal}(\hbf)\right)\leq \frac{1}{|S_{(\hbf,0)}|}\left[\KL(Q||P)+\ln\left(\frac{2\sqrt{m}}{\delta}\right)\right]\right) \geq 1-\delta\,.
&\qedhere
\end{align*}
\end{proof}

\triplesingle*

\begin{proof}
We obtain from the union bound, with probability at least $1-\delta$ for every $Q$ over $\Hcal$:
\begin{align*}
    &~~~~\tildeR_S(Q)\\
    &\geq R_\Dcal(Q)&\langle\textup{w.p. at least }1-\delta_1\rangle\\
    &= \underset{{(\mathbf{x}, \ybf)\sim \Dcal}}{\mathbb{E}}~\underset{\hbf'\sim Q}{\mathbb{E}} \ell(\hbf',\mathbf{x}, \ybf)&\\
    &= \underset{{(\mathbf{x}, \ybf)\sim \Dcal}}{\mathbb{P}}\Big(\ell(\hbf,\mathbf{x}, \ybf) = 1\Big)\underset{{(\mathbf{x}, \ybf)\sim \Dcal}}{\mathbb{E}}\left[\underset{\hbf'\sim Q}{\mathbb{E}} \ell(\hbf',\mathbf{x}, \ybf)~\bigg|~\ell(\hbf,\mathbf{x}, \ybf) = 1\right]+&\\
    &~~~~ \underset{{(\mathbf{x}, \ybf)\sim \Dcal}}{\mathbb{P}}\Big(\ell(\hbf,\mathbf{x}, \ybf) = 0\Big)\underset{{(\mathbf{x}, \ybf)\sim \Dcal}}{\mathbb{E}}\left[\underset{\hbf'\sim Q}{\mathbb{E}} \ell(\hbf',\mathbf{x}, \ybf)~\bigg|~\ell(\hbf,\mathbf{x}, \ybf) = 0\right]&\\
    &= \CondOne_{\Dcal}^Q(\hbf)\underset{{(\mathbf{x}, \ybf)\sim \Dcal}}{\mathbb{P}}\Big(\ell(\hbf,\mathbf{x}, \ybf) = 1\Big)+\CondZero_{\Dcal}^Q(\hbf)\underset{{(\mathbf{x}, \ybf)\sim \Dcal}}{\mathbb{P}}\Big(\ell(\hbf,\mathbf{x}, \ybf) = 0\Big)&\\
    &= \CondOne_{\Dcal}^Q(\hbf)\underset{{(\mathbf{x}, \ybf)\sim \Dcal}}{\mathbb{P}}\Big(\ell(\hbf,\mathbf{x}, \ybf) = 1\Big)+\CondZero_{\Dcal}^Q(\hbf)\Big(1-\underset{{(\mathbf{x}, \ybf)\sim \Dcal}}{\mathbb{P}}\Big(\ell(\hbf,\mathbf{x}, \ybf) = 1\Big)\Big)&\\
    &\geq \CondOne_{\Dcal}^Q(\hbf)\underset{{(\mathbf{x}, \ybf)\sim \Dcal}}{\mathbb{P}}\Big(\ell(\hbf,\mathbf{x}, \ybf) = 1\Big)+\tildezero_S^Q(\hbf)\Big(1-\underset{{(\mathbf{x}, \ybf)\sim \Dcal}}{\mathbb{P}}\Big(\ell(\hbf,\mathbf{x}, \ybf) = 1\Big)\Big)&\langle\textup{w.p. at least }1-\delta_2\rangle\\
    &\geq \tildeone_S^Q(\hbf)\underset{{(\mathbf{x}, \ybf)\sim \Dcal}}{\mathbb{P}}\Big(\ell(\hbf,\mathbf{x}, \ybf) = 1\Big)+\tildezero_S^Q(\hbf)\Big(1-\underset{{(\mathbf{x}, \ybf)\sim \Dcal}}{\mathbb{P}}\Big(\ell(\hbf,\mathbf{x}, \ybf) = 1\Big)\Big)&\langle\textup{w.p. at least }1-\delta_3\rangle\\
    &= (\tildeone_S^Q(\hbf)-\tildezero_S^Q(\hbf))\underset{{(\mathbf{x}, \ybf)\sim \Dcal}}{\mathbb{P}}\Big(\ell(\hbf,\mathbf{x}, \ybf) = 1\Big)+\tildezero_S^Q(\hbf)&\\
    &= (\tildeone_S^Q(\hbf)-\tildezero_S^Q(\hbf))\underset{(\mathbf{x}, \ybf)\sim \Dcal}{\mathbb{E}} \ell(\hbf,\mathbf{x}, \ybf)+\tildezero_S^Q(\hbf)&\\
    &= (\tildeone_S^Q(\hbf)-\tildezero_S^Q(\hbf))\risk_{\Dcal}(\hbf)+\tildezero_S^Q(\hbf)\,.&
\end{align*}
The main result is obtained by a simple rearrangement of the terms, given \mbox{$\tildeone_S^Q(\hbf) > \tildezero_S^Q(\hbf)$}.
\end{proof}

\triple*

\begin{proof}
The main result follows the same steps as \cref{cor:triple-single}'s proof.
\end{proof}

\categorical*

\begin{proof}
We simplify the inner content of $\CondZero_{\Dcal}^Q(\hbf)$ and $\CondOne_{\Dcal}^Q(\hbf)$ as such:
\begin{align*}
    \underset{{\mathbf{w}\sim \mathcal{C}\left(\mathbf{p}\right)}}{\mathbb{E}} \ell(\hbf_{\mathbf{w}},\mathbf{x}, \ybf)&=\underset{{\mathbf{w}\sim \mathcal{C}\left(\mathbf{p}\right)}}{\mathbb{E}} \mathbbm{1}\left\{\sum_{j=1}^n w_j\mathbbm{1}\left\{\fbf_j(\mathbf{x}) \neq \ybf\right\} \geq 0.5\right\}&\\
    &=\sum_{i=1}^n p_i \mathbbm{1}\{\mathbbm{1}\left\{\fbf_i(\mathbf{x}) \neq \ybf\right\} \geq 0.5\}&\\
    &=\sum_{i=1}^n p_i \mathbbm{1}\{\fbf_i(\mathbf{x}) \neq \ybf\}&\\
    &=e(h_\pbf, \xbf,\ybf)
\end{align*}
By choosing $\hbf:= \hbf_{\mathbf{p}}$ and given that 
\begin{align*}
&\ell(\hbf_{\mathbf{p}},\mathbf{x}, \ybf)=0\ \Leftrightarrow \ \sum_{i=1}^n p_i\mathbbm{1}\left\{\fbf_i(\mathbf{x}) \neq \ybf\right\} < 0.5\ \Leftrightarrow \  e(h_\pbf, \xbf,\ybf) < 0.5,\\
&\ell(\hbf_{\mathbf{p}},\mathbf{x}, \ybf)=1\ \Leftrightarrow \ \sum_{i=1}^n p_i\mathbbm{1}\left\{\fbf_i(\mathbf{x}) \neq \ybf\right\} \geq 0.5\ \Leftrightarrow \  e(h_\pbf, \xbf,\ybf) \geq 0.5,
\end{align*}
we obtain the desired results.    
\end{proof}

\partitioncategorical*

\begin{proof}
Recall that $e(h_\pbf, \xbf,\ybf) = \sum_{i=1}^n p_i\mathbbm{1}\left\{\fbf_i(\mathbf{x}) \neq \ybf\right\}$.
Thus, for any $(\mathbf{x}, \ybf)\in\Xcal\times\Ycal$, there exists $\mathbf{i}\in\{0,1\}^n$ such that $e(h_\pbf, \xbf,\ybf) = \mathbf{i}\cdot\mathbf{p}$.
\begin{align*}
\CondOne_{\Dcal}^{\mathcal{C}(\mathbf{p})}(\hbf_{\mathbf{p}}) &= \underset{(\mathbf{x}, \ybf)\sim \Dcal}{\mathbb{E}}\left[e(h_\pbf, \xbf,\ybf)~|~e(h_\pbf, \xbf,\ybf)\geq 0.5\right]\\
&\geq \underset{(\mathbf{x}, \ybf)\in\Xcal\times\Ycal}{\min}\left[e(h_\pbf, \xbf,\ybf)~|~e(h_\pbf, \xbf,\ybf)\geq 0.5\right]\\
&\geq \underset{\mathbf{i}\in\{0,1\}^n}{\min}\left[\mathbf{i}\cdot\mathbf{p}~|~\mathbf{i}\cdot\mathbf{p}\geq 0.5\right]\\
&= \max\left(\sum_{p\in\mathbf{p}_1}p, \sum_{p\in\mathbf{p}_2}p\right).
\end{align*}
The last equality comes from the fact that $\min_{\mathbf{i}\in\{0,1\}^n}\left[\mathbf{i}\cdot\mathbf{p}~|~\mathbf{i}\cdot\mathbf{p}\geq 0.5\right]$ is a reformulation of the partition problem \citep{10.1093/oso/9780195177374.003.0012}.
\end{proof}

\dirichlet*

\begin{proof}
We simplify the inner content of $\CondZero_{\Dcal}^Q(\hbf)$ and $\CondOne_{\Dcal}^Q(\hbf)$ as such:
\begin{align*}
    \underset{{\mathbf{w}\sim D\left(\mathbf{p}\right)}}{\mathbb{E}} \ell(\hbf_{\mathbf{w}},\mathbf{x}, \ybf)&=\underset{{\mathbf{w}\sim D\left(\mathbf{p}\right)}}{\mathbb{E}} \mathbbm{1}\left\{\sum_{j=1}^n w_j\mathbbm{1}\left\{\fbf_j(\mathbf{x}) \neq \ybf\right\} \geq \frac{||\wbf||_1}{2}\right\}&\\
    &=I_{0.5}\left(||\pbf||_1-e(h_\pbf, \xbf,\ybf), e(h_\pbf, \xbf,\ybf)\right).&\langle\textup{See \cite{DBLP:conf/nips/ZantedeschiVMEH21}}\rangle
\end{align*}
By choosing $\hbf:= \hbf_{\mathbf{p}}$ and given that 
\begin{align*}
&\ell(\hbf_{\mathbf{p}},\mathbf{x}, \ybf)=0\ \Leftrightarrow \ \sum_{i=1}^n p_i\mathbbm{1}\left\{\fbf_i(\mathbf{x}) \neq \ybf\right\} < \frac{||\pbf||_1}{2}\ \Leftrightarrow \  e(h_\pbf, \xbf,\ybf) < \frac{||\pbf||_1}{2},\\
&\ell(\hbf_{\mathbf{p}},\mathbf{x}, \ybf)=1\ \Leftrightarrow \ \sum_{i=1}^n p_i\mathbbm{1}\left\{\fbf_i(\mathbf{x}) \neq \ybf\right\} \geq \frac{||\pbf||_1}{2}\ \Leftrightarrow \  e(h_\pbf, \xbf,\ybf) \geq \frac{||\pbf||_1}{2},
\end{align*}
we obtain the desired results.    
\end{proof}

\partitiondirichlet*

\begin{proof}
We recall that $e(h_\pbf, \xbf,\ybf) = \sum_{i=1}^n p_i\mathbbm{1}\left\{\fbf_i(\mathbf{x}) \neq \ybf\right\}$ and thus that for any $(\mathbf{x}, \ybf)\in\Xcal\times\Ycal$, there exists $\mathbf{i}\in\{0,1\}^n$ such that $e(h_\pbf, \xbf,\ybf) = \mathbf{i}\cdot\mathbf{p}$.
\begin{align}\label{eq:proof_dir}
\CondOne_{\Dcal}^{D(\mathbf{p})}(\hbf_{\mathbf{p}}) &=\underset{(\mathbf{x}, \ybf)\sim \Dcal}{\mathbb{E}}\left[I_{0.5}\left(||\pbf||_1-e(h_\pbf, \xbf,\ybf), e(h_\pbf, \xbf,\ybf)\right)\Big|e(h_\pbf, \xbf,\ybf)\geq \frac{||\pbf||_1}{2}\right]\nonumber\\
&\geq \underset{(\mathbf{x}, \ybf)\in\Xcal\times\Ycal}{\min}\left[I_{0.5}\left(||\pbf||_1-e(h_\pbf, \xbf,\ybf), e(h_\pbf, \xbf,\ybf)\right)\Big|e(h_\pbf, \xbf,\ybf)\geq \frac{||\pbf||_1}{2}\right]\nonumber\\
&\geq \underset{\mathbf{i}\in\{0,1\}^n}{\min}\left[I_{0.5}\left(||\pbf||_1-\mathbf{i}\cdot\mathbf{p}, \mathbf{i}\cdot\mathbf{p}\right)\Big|\mathbf{i}\cdot\mathbf{p}\geq \frac{||\pbf||_1}{2}\right].
\end{align}
Since $I_{0.5}(\cdot,\cdot)$ is decreasing in its first argument and increasing in its second one, we have
$$
\underset{\mathbf{i}\in\{0,1\}^n}{\argmin}\left[I_{0.5}\left(||\pbf||_1-\mathbf{i}\cdot\mathbf{p}, \mathbf{i}\cdot\mathbf{p}\right)\Big|\,\mathbf{i}\cdot\mathbf{p}\geq \frac{||\pbf||_1}{2}\right] = \underset{\mathbf{i}\in\{0,1\}^n}{\argmin}\left[\mathbf{i}\cdot\mathbf{p}~|~\mathbf{i}\cdot\mathbf{p}\geq \frac{||\pbf||_1}{2}\right].
$$
And so, substituting in \cref{eq:proof_dir}:
\begin{align}\label{eq:proof_dir_2}
    \CondOne_{\Dcal}^{D(\mathbf{p})}(\hbf_{\mathbf{p}}) &\geq \underset{\mathbf{i}\in\{0,1\}^n}{\min}\left[I_{0.5}\left(||\pbf||_1-\mathbf{i}\cdot\mathbf{p}, \mathbf{i}\cdot\mathbf{p}\right)\Big|\mathbf{i}\cdot\mathbf{p}\geq \frac{||\pbf||_1}{2}\right]\nonumber \\
    &=I_{0.5}\left(||\pbf||_1-\mathbf{i}^*\cdot\mathbf{p}, \mathbf{i}^*\cdot\mathbf{p}\right),
\end{align}
with
$$
\ibf^* = \underset{\mathbf{i}\in\{0,1\}^n}{\argmin}\left[\mathbf{i}\cdot\mathbf{p}~|~\mathbf{i}\cdot\mathbf{p}\geq \frac{||\pbf||_1}{2}\right].
$$
Now, \cref{prop:partition_categorical} informed us that
$$
\ibf^*\cdot\pbf = \max\left(\sum_{p\in\mathbf{p}_1}p, \sum_{p'\in\mathbf{p}_2}p'\right)=\overset{\sim}{\pbf}\,,
$$
where $(\pbf_1,\pbf_2)$ is the result of the partition problem applied to $\pbf$. Thus, substituting in \cref{eq:proof_dir_2}:
\begin{align*}
\CondOne_{\Dcal}^{D(\mathbf{p})}(\hbf_{\mathbf{p}}) &\geq I_{0.5}\left(||\pbf||_1-\mathbf{i}^*\cdot\mathbf{p}, \mathbf{i}^*\cdot\mathbf{p}\right)\\
&= I_{0.5}\left(||\pbf||_1-\overset{\sim}{\pbf}, \overset{\sim}{\pbf}\right).
\end{align*}
\end{proof}

\binarygaussian*
\begin{proof}
We simplify the inner content of $\CondZero_{\Dcal}^Q(h)$ and $\CondOne_{\Dcal}^Q(h)$ as such:
\begin{align*}
    \underset{{\mathbf{w}\sim \Ncal\left(\mathbf{p}, \I\right)}}{\mathbb{E}} \ell(\hbf_{\mathbf{w}},\mathbf{x}, y)&=\underset{{\mathbf{w}\sim \Ncal\left(\mathbf{p}, \I\right)}}{\mathbb{E}} \mathbbm{1}\left\{y \neq \underset{\hat{y}\in\mathcal{Y}}{\argmax} \sum_{j=1}^n w_j \mathbbm{1}\{f_j(\mathbf{x}) = \hat{y}\}\right\}&\\
    &=\underset{{\mathbf{w}\sim \Ncal\left(\mathbf{p}, \I\right)}}{\mathbb{E}}\frac{1}{2}\left(1- y~\sgn(\mathbf{w}^{\top}\fbf(\xbf))\right)&\\
    &=\Phi\left(y\frac{\pbf\cdot\fbf(\xbf)}{||\fbf(\xbf)||}\right).\qquad\qquad\qquad\qquad\qquad\qquad\langle\textup{See \cite{DBLP:conf/nips/LangfordS02}}\rangle
\end{align*}
By choosing $h := \hbf_{\mathbf{p}}$:
\begin{align*}
    \CondOne_{\Dcal}^{\Ncal(\mathbf{p}, \I)}(\hbf_{\mathbf{p}})&=\underset{(\mathbf{x}, y)\sim \Dcal}{\mathbb{E}} \left[\underset{{\mathbf{w}\sim \Ncal\left(\mathbf{p}, \I\right)}}{\mathbb{E}} \ell(\hbf_{\mathbf{w}},\mathbf{x}, y)~\Big|~\ell(\hbf_{\mathbf{p}},\mathbf{x}, y)=1\right]&\\
    &=\underset{(\mathbf{x}, y)\sim \Dcal}{\mathbb{E}} \left[\Phi\left(y\frac{\pbf\cdot\fbf(\xbf)}{||\fbf(\xbf)||}\right)~\Big|~\frac{1}{2}\left(1- y~\sgn(\mathbf{p}\cdot\fbf(\xbf))\right)=1\right]&\\
    &=\underset{(\mathbf{x}, y)\sim \Dcal}{\mathbb{E}} \left[\Phi\left(\frac{y~\pbf\cdot\fbf(\xbf)}{||\fbf(\xbf)||}\right)~\Big|~y~\mathbf{p}\cdot\fbf(\xbf)\leq 0\right]&\\
    &=\underset{(\mathbf{x}, y)\sim \Dcal}{\mathbb{E}} \left[\Phi\left(\frac{-|y~\pbf\cdot\fbf(\xbf)|}{||\fbf(\xbf)||}\right)~\Big|~y~\mathbf{p}\cdot\fbf(\xbf)\leq 0\right]&\\
    &=\underset{(\mathbf{x}, y)\sim \Dcal}{\mathbb{E}} \left[1-\Phi\left(\frac{|\pbf\cdot\fbf(\xbf)|}{||\fbf(\xbf)||}\right)~\Big|~y~\mathbf{p}\cdot\fbf(\xbf)\leq 0\right]&\\
    &=1-\underset{(\mathbf{x}, y)\sim \Dcal}{\mathbb{E}} \left[\Phi\left(\frac{|\pbf\cdot\fbf(\xbf)|}{||\fbf(\xbf)||}\right)~\Big|~y~\mathbf{p}\cdot\fbf(\xbf)\leq 0\right].&
\end{align*}
\begin{align*}
\CondZero_{\Dcal}^{\Ncal(\mathbf{p}, \I)}(\hbf_{\mathbf{p}})&=\underset{(\mathbf{x}, y)\sim \Dcal}{\mathbb{E}} \left[\underset{{\mathbf{w}\sim \Ncal\left(\mathbf{p}, \I\right)}}{\mathbb{E}} \ell(\hbf_{\mathbf{w}},\mathbf{x}, y)~\Big|~\ell(\hbf_{\mathbf{p}},\mathbf{x}, y)=0\right]&\\
    &=\underset{(\mathbf{x}, y)\sim \Dcal}{\mathbb{E}} \left[\Phi\left(y\frac{\pbf\cdot\fbf(\xbf)}{||\fbf(\xbf)||}\right)~\Big|~\frac{1}{2}\left(1- y~\sgn(\mathbf{p}\cdot\fbf(\xbf))\right)=0\right]&\\
    &=\underset{(\mathbf{x}, y)\sim \Dcal}{\mathbb{E}} \left[\Phi\left(\frac{y~\pbf\cdot\fbf(\xbf)}{||\fbf(\xbf)||}\right)~\Big|~y~\mathbf{p}\cdot\fbf(\xbf) > 0\right]&\\
    &=\underset{(\mathbf{x}, y)\sim \Dcal}{\mathbb{E}} \left[\Phi\left(\frac{|y~\pbf\cdot\fbf(\xbf)|}{||\fbf(\xbf)||}\right)~\Big|~y~\mathbf{p}\cdot\fbf(\xbf) > 0\right]&\\
    &=\underset{(\mathbf{x}, y)\sim \Dcal}{\mathbb{E}} \left[\Phi\left(\frac{|\pbf\cdot\fbf(\xbf)|}{||\fbf(\xbf)||}\right)~\Big|~y~\mathbf{p}\cdot\fbf(\xbf) > 0\right].&
\end{align*}
\end{proof}

\partitionbinarygaussian*

\begin{proof}
Since every base classifier has a prediction in $\{-1,+1\}$, for every $\xbf$, we have $||\fbf(\xbf)|| = \sqrt{n}$. We have:
\begin{align*}
\CondOne_{\Dcal}^{\Ncal(\mathbf{p}, \I)}(\hbf_{\mathbf{p}})&=1-\underset{(\mathbf{x}, y)\sim \Dcal}{\mathbb{E}} \left[\Phi\left(\frac{|\pbf\cdot\fbf(\xbf)|}{\sqrt{n}}\right)~\Big|~y~\mathbf{p}\cdot\fbf(\xbf) \leq 0\right]\\
&\geq 1-\underset{\mathbf{x}\in\Xcal}{\max}~\Phi\left(\frac{|\pbf\cdot\fbf(\xbf)|}{\sqrt{n}}\right)\\
&\geq 1- \Phi\left(\frac{\underset{\mathbf{x}\in\Xcal}{\min}|\pbf\cdot\fbf(\xbf)|}{\sqrt{n}}\right)&\langle~\Phi\text{ is strictly decreasing. }\rangle\\
&\geq 1- \Phi\left(\frac{\underset{\mathbf{i}\in\{-1,1\}^n}{\min}|\pbf\cdot\ibf|}{\sqrt{n}}\right).
\end{align*}
Finally, note that 
$$
\underset{\mathbf{i}\in\{-1,1\}^n}{\min}|\pbf\cdot\ibf| = \min_{\mathbf{p}_1,\mathbf{p}_2}\left\{\big|\sum_{p\in\mathbf{p}_1}p - \sum_{p'\in\mathbf{p}_2}p'\Big|:\{\pbf_1,\pbf_2\}\textup{ is a partition of }\pbf\right\},
$$
which corresponds to the objective of the partition problem. Plugging that into the development yields the main result.
\end{proof}

\maxbinarygaussian*

\begin{proof}
Since every base classifier has a prediction in $\{-1,+1\}$, for every $\xbf$, we have $||\fbf(\xbf)|| = \sqrt{n}$ and \mbox{$\max_{\mathbf{x}\in\Xcal}|\pbf\cdot\fbf(\xbf)| = ||\pbf||_1$}. Thus:
\begin{align*}
\CondZero_{\Dcal}^{\Ncal(\mathbf{p}, \I)}(\hbf_{\mathbf{p}})&=\underset{(\mathbf{x}, y)\sim \Dcal}{\mathbb{E}} \left[\Phi\left(\frac{|\pbf\cdot\fbf(\xbf)|}{\sqrt{n}}\right)~\Big|~y~\mathbf{p}\cdot\fbf(\xbf) > 0\right]\\
&\geq \underset{\mathbf{x}\in\Xcal}{\min}~\Phi\left(\frac{|\pbf\cdot\fbf(\xbf)|}{\sqrt{n}}\right)\\
&\geq  \Phi\left(\frac{\underset{\mathbf{x}\in\Xcal}{\max}|\pbf\cdot\fbf(\xbf)|}{\sqrt{n}}\right)&\langle~\Phi\text{ is strictly decreasing. }\rangle\\
&\geq  \Phi\left(\frac{||\pbf||_1}{\sqrt{n}}\right).
\end{align*}
\end{proof}

\begin{lemma}\label{lem:gaussian_multi_1}
    If $\wbf\sim\mathcal{N}(\boldsymbol\mu,\I)$, where $\mubf\in\Rbb^n$, and $\mathbf{a}_i\in\mathbb{R}^n$ for $i\in\{1,\dots,k\}$, then 
    $$
    (\wbf^\top \mathbf{a}_1,\dots,\wbf^\top \mathbf{a}_k)\sim\mathcal{N}(\hatmubf,\hatSigma),
    $$
    where $\hatmubf = [\boldsymbol\mu^\top \mathbf{a}_1,\dots,\boldsymbol\mu^\top \mathbf{a}_k]$ and $\hatSigma_{i,j} = \mathbf{a}_i^\top \mathbf{a}_j$.
\end{lemma}

\medskip

\gaussian*

\begin{proof}
First notice that 
$$
\mathbbm{1}\left\{\ybf \neq \underset{\hat{\ybf}\in\mathcal{Y}}{\argmax} \sum_{j=1}^n p_j \mathbbm{1}\{\fbf_j(\mathbf{x}) = \hat{\ybf}\}\right\} = 1-\ybf\cdot\hmax(p_1\fbf_1(\mathbf{x})+\dots+p_n\fbf_n(\mathbf{x})),
$$
where
$$
\forall \abf\in\mathbb{R}^n, i\in\{1,\dots,n\}~:~\hmax_i(\abf) =
    \begin{cases}
      1 & \text{if } ~a_i>a_j\forall j\neq i,\\
      0 & \text{otherwise}.
    \end{cases}
$$
In words, $\hmax$ computes the hard-max of an entry vector, returning a one-hot vector. Without loss of generality, suppose $\ybf$ is a one-hot vector with value $1$ at position $i$. We denote
$$
\fbf_{:,i}(\cdot) = [f_{1,i}(\cdot),\dots,f_{n,i}(\cdot)]\,.
$$
Then,
\allowdisplaybreaks[4]{
\begin{align*}
    \underset{{\mathbf{w}\sim \Ncal\left(\mathbf{p}, \I\right)}}{\mathbb{E}} \ell(\hbf_{\mathbf{w}},\mathbf{x}, \ybf)&=\underset{{\mathbf{w}\sim \Ncal\left(\mathbf{p}, \I\right)}}{\mathbb{E}} \mathbbm{1}\left\{\ybf \neq \underset{\hat{\ybf}\in\mathcal{Y}}{\argmax} \sum_{j=1}^n w_j \mathbbm{1}\{\fbf_j(\mathbf{x}) = \hat{\ybf}\}\right\}&\\
    &=\underset{{\mathbf{w}\sim \Ncal\left(\mathbf{p}, \I\right)}}{\mathbb{E}}\ybf\cdot\hmax(w_1\fbf_1(\mathbf{x})+\dots+w_n\fbf_n(\mathbf{x}))\\
    &=1-\ybf\cdot\underset{{\mathbf{w}\sim \Ncal\left(\mathbf{p}, \I\right)}}{\mathbb{E}}\hmax(w_1\fbf_1(\mathbf{x})+\dots+w_n\fbf_n(\mathbf{x}))\\
    &=1-\sum_{i=1}^n\mathbbm{1}\{y_i = 1\}\underset{{\mathbf{w}\sim \Ncal\left(\mathbf{p}, \I\right)}}{\mathbb{E}}\hmax_i(w_1\fbf_1(\mathbf{x})+\dots+w_n\fbf_n(\mathbf{x}))\\
    &=1-\sum_{i=1}^n\mathbbm{1}\{y_i = 1\}\underset{{\mathbf{w}\sim \Ncal\left(\mathbf{p}, \I\right)}}{\mathbb{P}}\left(\hmax_i(w_1\fbf_1(\mathbf{x})+\dots+w_n\fbf_n(\mathbf{x}))=1\right)\\
    &=1-\sum_{i=1}^n\mathbbm{1}\{y_i = 1\}\underset{{\mathbf{w}\sim \Ncal\left(\mathbf{p}, \I\right)}}{\mathbb{P}}\left(\hmax_i(\wbf\cdot\fbf_{:,1}(\mathbf{x}),\dots,\wbf\cdot\fbf_{:,n}(\mathbf{x}))=1\right)\\
    &=1-\sum_{i=1}^n\mathbbm{1}\{y_i = 1\}\underset{{\mathbf{w}\sim \Ncal\left(\mathbf{p}, \I\right)}}{\mathbb{P}}\left(\bigwedge_{j\neq i} \wbf\cdot\fbf_{:,i}(\mathbf{x}) > \wbf\cdot\fbf_{:,j}(\mathbf{x})\right)\\
    &=1-\sum_{i=1}^n\mathbbm{1}\{y_i = 1\}\underset{{\mathbf{w}\sim \Ncal\left(\mathbf{p}, \I\right)}}{\mathbb{P}}\left(\bigwedge_{j\neq i} \wbf\cdot(\fbf_{:,j}(\mathbf{x})-\fbf_{:,i}(\mathbf{x})) \leq 0\right)\\
    &=1-\sum_{i=1}^n\mathbbm{1}\{y_i = 1\}F_{Z_i}(\mathbf{0})\,,
\end{align*}
}
 where $F$ is the cumulative distribution function. The final result is obtained by using \cref{lem:gaussian_multi_1} to find the distribution of $Z_i$.
\end{proof}

\newpage

\section{An attemp at new results}

\subsection{The limitation of stochastic majority votes}

Le risque PAC-Bayésien, avec un posterior Categoriel, est pire que le risque du meilleur classifieur de base...
\begin{align*}
R_{\Dcal}(\Ccal(\pbf))&=\underset{(\xbf,\ybf)\sim\Dcal}{\Esp}\left[\underset{\wbf\sim \Ccal(\pbf)}{\Esp} R_{\Dcal}(h_{\wbf})\right]\\
&=\underset{(\xbf,\ybf)\sim\Dcal}{\Esp}\left[\underset{{\mathbf{w}\sim \mathcal{C}\left(\mathbf{p}\right)}}{\mathbb{E}} \mathbbm{1}\left\{\sum_{i=1}^n w_i\mathbbm{1}\left\{\fbf_i(\mathbf{x}) \neq \ybf\right\} \geq 0.5\right\}\right]\\
&=\underset{(\xbf,\ybf)\sim\Dcal}{\Esp}\left[\sum_{i=1}^n p_i \mathbbm{1}\{\mathbbm{1}\left\{\fbf_i(\mathbf{x}) \neq \ybf\right\} \geq 0.5\}\right]\\
&=\underset{(\xbf,\ybf)\sim\Dcal}{\Esp}\left[\sum_{i=1}^n p_i \mathbbm{1}\left\{\fbf_i(\mathbf{x}) \neq \ybf\right\}\right]\\
&=\sum_{i=1}^n p_i\underset{(\xbf,\ybf)\sim\Dcal}{\Esp}\left[ \mathbbm{1}\left\{\fbf_i(\mathbf{x}) \neq \ybf\right\}\right]\\
&=\sum_{i=1}^n p_i R_\Dcal(\fbf_i)\\
&\geq\min_{i\in\{1,\dots,n\}} R_\Dcal(\fbf_i)\\
\end{align*}

\subsection{My new bound}

Let $\Dcal$ be a data distribution defined over $\Xcal\times\Ycal$. Let $f_i~:~\Xcal\rightarrow\Ycal$.
Then:

\begin{align*}
&\underset{(\xbf,\ybf)\sim\Dcal}{\Esp}\left[R_{(\xbf,\ybf)}(Q)~\Bigg|~R_{(\xbf,\ybf)}(h) > 0.5\right]\\
=&\sum_{v\in\Vcal} v\cdot\underset{(\xbf,\ybf)\sim\Dcal}{\Pbb}\left(R_{(\xbf,\ybf)}(Q)=v~\Bigg|~R_{(\xbf,\ybf)}(h) > 0.5\right)\\
\end{align*}

Choosing similar models, we obtain

\newpage

\section{Details about the experimental section}
\label{appendix:experimental}
\renewcommand{\arraystretch}{1.1}
\begin{table}[ht]
    \centering
    \caption{Details of the datasets used in the experimental section coming from either the UCI datasets repository \citep{Dua:2019}, the LIBSVM library \citep{chang2011libsvm}, or Zalando \citep{DBLP:journals/corr/abs-1708-07747}. $d$ is the number of features, $k$ the number of classes and $m$ the total number of examples.}
    \label{tab:datasets_overview}
    \setlength{\tabcolsep}{4pt}
    {\small
    \begin{tabular}{lllllllll}
    \toprule
    \# & Dataset & Full name & Source & $d$ & $k$ & $m$\\
    \midrule
    1 & ADULT & Adult & LIBSVM & 123 & 2 & 32 561 \\
    2 & CODRNA & CodRNA & LIBSVM & 8 & 2 & 59 535 \\
    3 & HABER & Haberman & UCI & 3 & 2 & 306 \\
    4 & MUSH & Mushroom & UCI & 22 & 2 & 8124 \\
    5 & PHIS & Phishing & LIBSVM & 68 & 2 & 2456 \\
    6 & SVMG & Svmguide1 & LIBSVM & 4 & 2 & 7089 \\
    7 & TTT & TicTacToe & UCI & 9 & 2 & 958 \\
    \hline
    8 & FASHION & Fashion-MNIST & Zalando & 784 & 10 & 70 000 \\
    9 & MNIST & MNIST & LIBSVM & 784 & 10 & 70 000 \\
    10 & PENDIG & Pendigits & UCI & 9 & 10 & 12 992 \\
    11 & PROTEIN & Protein & LIBSVM & 357 & 3 & 24 387 \\
    12 & SENSOR & Sensorless & LIBSVM & 48 & 11 & 58 509 \\
    \bottomrule
    \end{tabular}
    }
\end{table}

\paragraph{Hyperparameters, datasets and resources.}

We used an NVIDIA GeForce RTX 2080 Ti graphics card for the experiments. We used a batch size equal to 1024, and a learning rate equal to 0.1 with a scheduler reducing this parameter by a factor of 10 with a 2-epoch patience. The maximal number of epochs is set to 100, and patience is set to 25 for performing early stopping. 

\cref{tab:datasets_overview} presents the datasets used in the experiments.

\paragraph{Heuristics running times.}

\cref{tab:running_time} presents the average running time of every heuristic, as well as the partition problem solving for the selected models of \cref{section:experiments} (the partition problem algorithm being used many times in a single model optimization loop). Note that the average running time is independent of the distribution (either Categorical, Dirichlet, or Gaussian) used to model the prior and posterior PAC-Bayesian distributions in our training loop. 

From \cref{tab:running_time}, it is clear that the computational overhead of solving the partition problem is negligible. 
We have not encountered failure for the algorithm to solve the partition problem, even though we used a time cap of five seconds for the algorithm to run. 
If the time cap was reached, or the algorithm failed to provide a solution, one could rely on hard lower and upper bounds $\CondOne_{\Dcal}^{Q}(\hbf_{\mathbf{p}}) \geq 0.5$ and \mbox{$\CondZero_{\Dcal}^{Q}(\hbf_{\mathbf{p}}) \geq 0$} discussed in \cref{section:DRND_mv}. 
On the other hand, the running time of the various heuristics is small for most datasets, though Heuristic \#2 can run for over a minute on the Adult dataset. To address this, one could easily add a time limit on the use of the various heuristics.

\begin{table}
    \centering
    \caption{Average and standard deviation of running time of the partition problem algorithm and the computations of every heuristic for the selected models (over five random seeds).}
    \label{tab:running_time}
    \setlength{\tabcolsep}{4pt}
    {\small
    \begin{tabular}{llcccc}
    \toprule
    \# & Dataset & Partition problem solving  (ms) & Heuristic \#1 (s) & Heuristic \#2 (s) & Heuristic \#3 (s) \\
    \midrule
    1 & ADULT & 124.4 $\pm$ 11.2 & 12.3 $\pm$ 2.3 & 67.1 $\pm$ 20.8 & 16.7 $\pm$ 1.0\\
    2 & CODRNA & 23.7 $\pm$ 7.5 & 14.2 $\pm$ 4.6 & 1.9 $\pm$ 0.2 & 10.2 $\pm$ 1.8\\
    3 & HABER & 17.3 $\pm$ 4.3 & 2.3 $\pm$ 0.1 & 0.4 $\pm$ 0.1 & 1.9 $\pm$ 0.3\\
    4 & MUSH & 42.1 $\pm$ 8.8 & 2.9 $\pm$ 0.4 & 1.1 $\pm$ 0.2 & 3.1 $\pm$ 0.3\\
    5 & PHIS & 12.3 $\pm$ 4.0 & 4.1 $\pm$ 1.0 & 6.2 $\pm$ 1.1 & 3.4 $\pm$ 0.7\\
    6 & SVMG & 16.0 $\pm$ 4.6 & 2.8 $\pm$ 0.1 & 0.4 $\pm$ 0.0 & 2.1 $\pm$ 0.4\\
    7 & TTT & 26.5 $\pm$ 5.8 & 2.4 $\pm$ 0.3 & 4.0 $\pm$ 0.5 & 2.0 $\pm$ 0.2\\
    \hline
    8 & FASHION & 11.8 $\pm$ 2.4 & 15.5 $\pm$ 1.5 & 5.8 $\pm$ 0.8 & 13.2 $\pm$ 3.0\\
    9 & MNIST & 13.6 $\pm$ 2.3 & 17.8 $\pm$ 3.4 & 7.6 $\pm$ 2.2 & 14.1 $\pm$ 2.1\\
    10 & PENDIG & 10.9 $\pm$ 1.8 & 5.6 $\pm$ 0.6 & 1.8 $\pm$ 0.6 & 4.3 $\pm$ 0.3\\
    11 & PROTEIN & 10.3 $\pm$ 0.9 & 6.9 $\pm$ 0.9 & 6.8 $\pm$ 1.1 & 5.4 $\pm$ 2.0\\
    12 & SENSOR & 14.5 $\pm$ 3.7 & 10.3 $\pm$ 0.9 & 4.5 $\pm$ 0.5 & 8.6 $\pm$ 0.6\\
    \bottomrule
    \end{tabular}
    }
\end{table}

\paragraph{Partition problem solving running times.}

As mentioned in \cref{section:optimization}, we rely on a complete greedy algorithm solving an integer linear program to solve partition problem instances. We monitored the running time of the algorithm during optimization of our approach on the SVMG dataset, for the Categorical distribution and the stump base classifiers, for various numbers of base classifiers. Indeed, the only parameter affecting the partition problem algorithm running time is the number of base classifiers, so the results can be generalized across every setup considered in \cref{section:experiments}. The observed running times are shown in \cref{fig:partition_problem_running_time}. The runs were done using an Intel(R) Core(TM) i7-10510U CPU.

\begin{figure}
    \centering
    \includegraphics[width=0.75\linewidth]{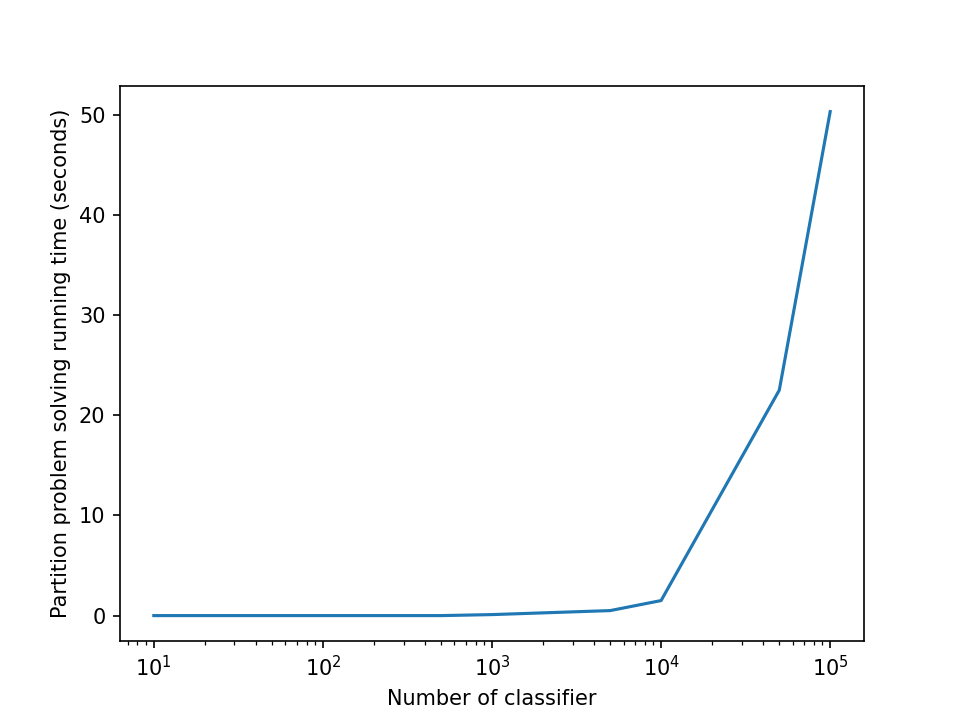}
    \caption{Average running time of the considered partition problem-solving algorithm throughout a complete optimization using our proposed method for various numbers of base classifiers (over five random seeds).}
    \label{fig:partition_problem_running_time}
\end{figure}

The number of base classifiers used during the experiments reported in \cref{section:experiments} was at most 1250 (for the ADULT dataset); \cref{fig:partition_problem_running_time} clearly shows why, during the experiments, the running time of the partition problem solving was negligible. The running time actually is negligible for up to about $10^4$ base classifiers; at this point, the running time starts growing, attaining about 50 seconds at $10^5$ base classifiers. 

\paragraph{Chosen distribution.}

\cref{tab:chosen_distribution} presents the selected distribution for both our proposed method and the baselines for the task presented in \cref{section:experiments}. Overall, it is clear that considering the many possible distributions is relevant, from our approach's perspective, or from that of the First Order or the Second Order baseline, even though the Categorical distribution is chosen most of the time.

\setlength{\tabcolsep}{1.75pt}
\renewcommand{\arraystretch}{1.05}
\begin{table*}
    \caption{Selected distributions for our approach and a few baselines across different types of predictors.}
    \centering
    \resizebox{\textwidth}{!}{%
    \begin{tabular}{|c||c|c|c|c|c||c|}
        \hline
        \multirow{2}{*}{Task} & First Order & Second Order & Binomial Bound & C-Bound & VC-dimension & \textbf{Partition Bound}\\
        & \textcolor{ForestGreen}{C}, or \textcolor{red}{G} & \textcolor{ForestGreen}{C}, or \textcolor{red}{G} & \textcolor{ForestGreen}{C}, or \textcolor{red}{G} & \textcolor{ForestGreen}{C} & \textcolor{ForestGreen}{C} & \textcolor{ForestGreen}{C}, \textcolor{blue}{D}, or \textcolor{red}{G} \\
\hline
ADULT & \textcolor{ForestGreen}{Categorical} & \textcolor{red}{Gaussian} & \textcolor{ForestGreen}{Categorical} & \textcolor{ForestGreen}{Categorical} & \textcolor{ForestGreen}{Categorical} & \textcolor{blue}{Dirichlet} \\ 
CODRNA & \textcolor{red}{Gaussian} & \textcolor{red}{Gaussian} & \textcolor{ForestGreen}{Categorical} & \textcolor{ForestGreen}{Categorical} & \textcolor{ForestGreen}{Categorical} & \textcolor{ForestGreen}{Categorical} \\ 
HABER & \textcolor{ForestGreen}{Categorical} & \textcolor{ForestGreen}{Categorical} & \textcolor{ForestGreen}{Categorical} & \textcolor{ForestGreen}{Categorical} & \textcolor{ForestGreen}{Categorical} & \textcolor{ForestGreen}{Categorical} \\ 
MUSH & \textcolor{ForestGreen}{Categorical} & \textcolor{ForestGreen}{Categorical} & \textcolor{ForestGreen}{Categorical} & \textcolor{ForestGreen}{Categorical} & \textcolor{ForestGreen}{Categorical} & \textcolor{ForestGreen}{Categorical} \\ 
PHIS & \textcolor{ForestGreen}{Categorical} & \textcolor{red}{Gaussian} & \textcolor{ForestGreen}{Categorical} & \textcolor{ForestGreen}{Categorical} & \textcolor{ForestGreen}{Categorical} & \textcolor{ForestGreen}{Categorical} \\ 
SVMGUIDE & \textcolor{ForestGreen}{Categorical} & \textcolor{ForestGreen}{Categorical} & \textcolor{ForestGreen}{Categorical} & \textcolor{ForestGreen}{Categorical} & \textcolor{ForestGreen}{Categorical} & \textcolor{ForestGreen}{Categorical} \\ 
TTT & \textcolor{ForestGreen}{Categorical} & \textcolor{ForestGreen}{Categorical} & \textcolor{ForestGreen}{Categorical} & \textcolor{ForestGreen}{Categorical} & \textcolor{ForestGreen}{Categorical} & \textcolor{ForestGreen}{Categorical} \\ 
\hline
\hline
\multirow{2}{*}{Task} & First Order & Second Order & Binomial Bound & C-Bound & VC-dimension & \textbf{Partition Bound}\\
        & \textcolor{ForestGreen}{C}, or \textcolor{red}{G} & \textcolor{ForestGreen}{C}, or \textcolor{red}{G} & \textcolor{ForestGreen}{C}, or \textcolor{red}{G} & \textcolor{ForestGreen}{C} & N/A & \textcolor{ForestGreen}{C}, or \textcolor{blue}{D}\\
\hline
FASHION & \textcolor{red}{Gaussian} & \textcolor{ForestGreen}{Categorical} & \textcolor{ForestGreen}{Categorical} & \textcolor{ForestGreen}{Categorical} & N/A & \textcolor{ForestGreen}{Categorical} \\ 
MNIST & \textcolor{red}{Gaussian} & \textcolor{ForestGreen}{Categorical} & \textcolor{ForestGreen}{Categorical} & \textcolor{ForestGreen}{Categorical} & N/A & \textcolor{ForestGreen}{Categorical} \\ 
PENDIGITS & \textcolor{red}{Gaussian} & \textcolor{ForestGreen}{Categorical} & \textcolor{ForestGreen}{Categorical} & \textcolor{ForestGreen}{Categorical} & N/A & \textcolor{blue}{Dirichlet} \\ 
PROTEIN & \textcolor{red}{Gaussian} & \textcolor{ForestGreen}{Categorical} & \textcolor{ForestGreen}{Categorical} & \textcolor{ForestGreen}{Categorical} & N/A & \textcolor{blue}{Dirichlet} \\ 
SENSORLESS & \textcolor{red}{Gaussian} & \textcolor{ForestGreen}{Categorical} & \textcolor{ForestGreen}{Categorical} & \textcolor{ForestGreen}{Categorical} & N/A & \textcolor{blue}{Dirichlet} \\ 
\hline
    \end{tabular}}
    \label{tab:chosen_distribution}
\end{table*}

\section{Ablation study}
\label{appendix:ablation}
We now present an ablation study regarding the various heuristics. \cref{fig:categorical,fig:dirichlet,fig:gaussian} respectively present the bounds obtained for the Categorical, the Dirichlet, and the Gaussian distributions for each task considered in \cref{section:experiments}. In each case, we display the bound obtained before the application of any heuristic; after the application of the first heuristic; after the application of the first two heuristics; and after the application of every heuristic. We also present the variation in observed test error as each heuristic is applied.

While the three heuristics do not systematically play a significant role (in many cases, the gains are marginal or absent), on some tasks they help improve the bound value. There does not seem to be any pattern regarding the role of each heuristic individually, yet they all seem to play an important role on some tasks. Most interestingly, the decay of the bound values has almost no impact on the reported test error, meaning that the heuristic might improve the bound value without impacting the model performance. This confirms that the heuristics can be applied at no cost, except for their respective running time.

\begin{figure}
    \centering
    \includegraphics[width=1\linewidth]{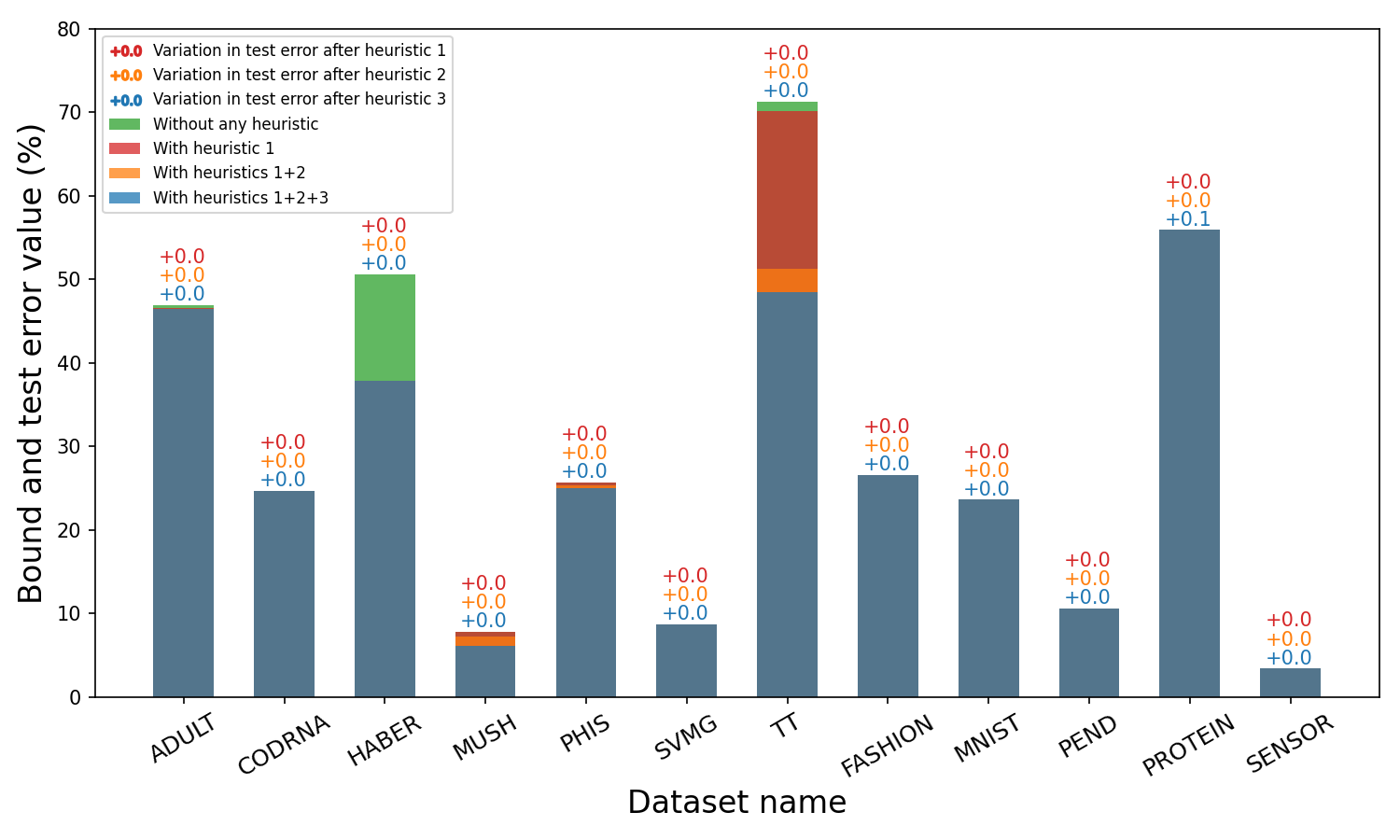}
    \caption{Evolution of the selected model bound and test error for the Categorical distribution as the various heuristics are applied.}
    \label{fig:categorical}
\end{figure}

\begin{figure}
    \centering
    \includegraphics[width=1\linewidth]{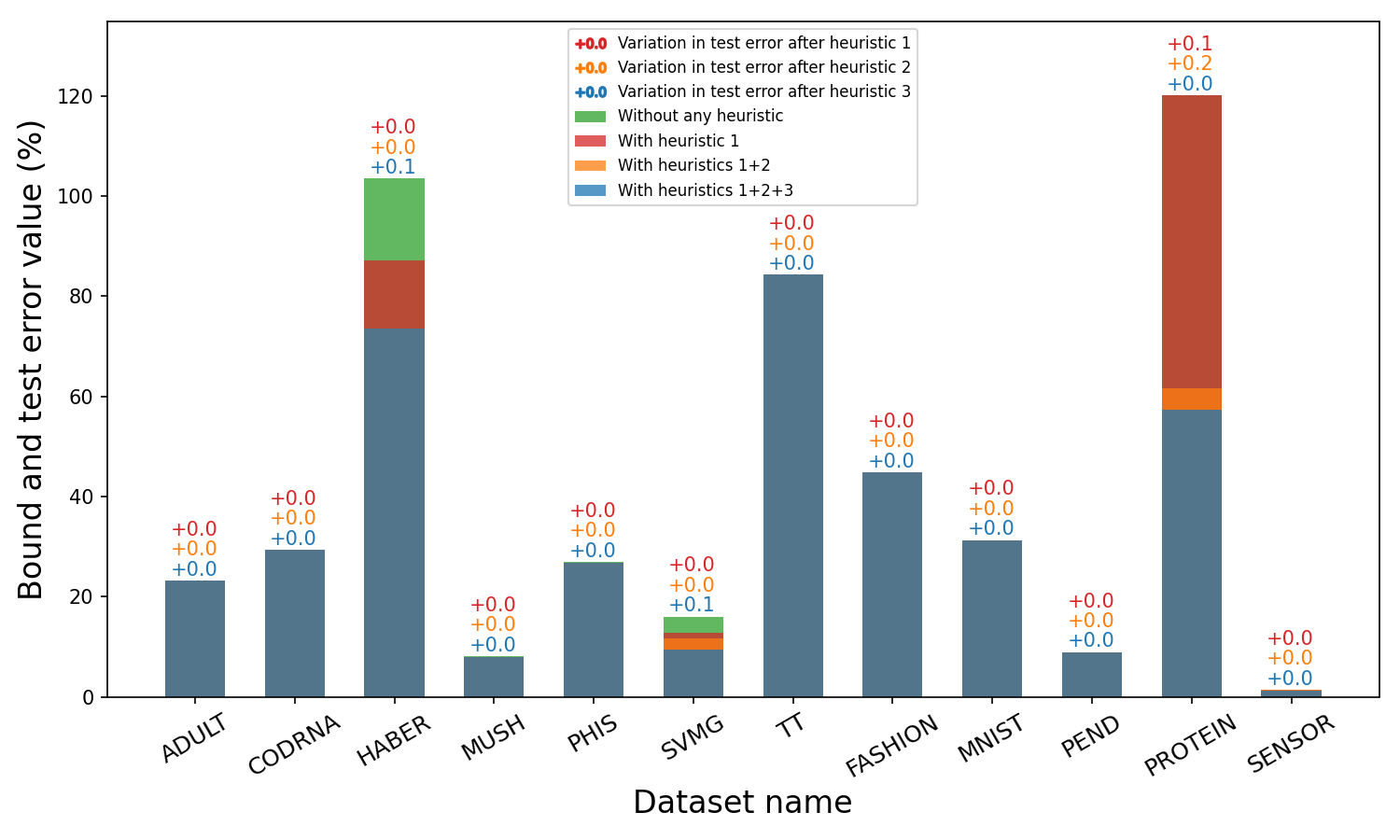}
    \caption{Evolution of the selected model bound and test error for the Dirichlet distribution as the various heuristics are applied.}
    \label{fig:dirichlet}
\end{figure}

\begin{figure}
    \centering
    \includegraphics[width=1\linewidth]{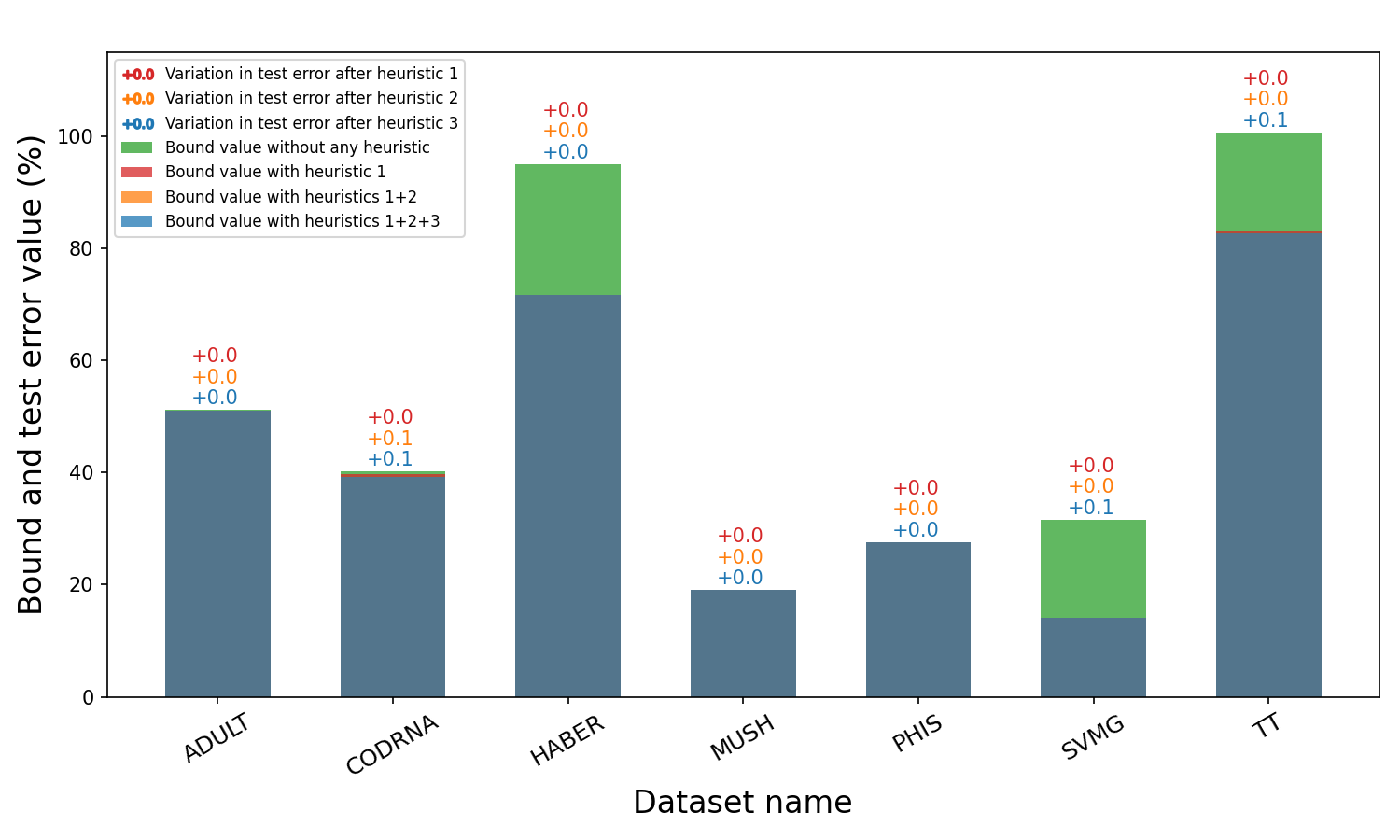}
    \caption{Evolution of the selected model bound and test error for the Gaussian distribution as the various heuristics are applied.}
    \label{fig:gaussian}
\end{figure}

\end{document}